\documentclass[final,3p,times]{elsarticle}

\usepackage{dsfont}

\usepackage{lineno} 
\modulolinenumbers[1]
\journal{Pattern Recognition}

\usepackage{array}

\usepackage{graphicx}
\graphicspath{{./pdf/}{./jpeg/}{./png/}}
\DeclareGraphicsExtensions{.png}
\usepackage{multirow}
\usepackage{amsmath}
\usepackage{amsfonts}
\usepackage{array}
\usepackage{comment}
\usepackage[caption=false,font=footnotesize,labelfont=sf,textfont=sf]{subfig}
\usepackage{fixltx2e}
\usepackage{stfloats}

\usepackage{algorithm}
\usepackage{algorithmic}

\usepackage{url}
\usepackage{color}

\DeclareMathOperator*{\argmax}{arg\,max}

\def\Circlearrowright{\ensuremath{%
  \rotatebox[origin=c]{90}{$\circlearrowright$}}}

\usepackage{amsmath}

\usepackage{amssymb}
\usepackage{amsthm}
\usepackage{color}

\usepackage{graphicx}

\usepackage[numbers]{natbib}

\newtheorem{theorem}{Theorem}
\newtheorem{lemma}{Lemma}
\newtheorem{proposition}{Proposition}

\begin{document}

\begin{frontmatter}



\title{Revisiting Classical Multiclass Linear Discriminant Analysis with a Novel Prototype-based Interpretable Solution}


\author[comp_dep]{Kamaledin Ghiasi-Shirazi}
\ead{k.ghiasi@um.ac.ir}
\address[comp_dep]{Department of Computer Engineering, Ferdowsi University of Mashhad (FUM), Mashhad, Khorasan Razavi, Iran}

\begin{abstract}
Linear discriminant analysis (LDA) is a fundamental method for feature extraction and dimensionality reduction. Despite having many variants, classical LDA has its own importance, as it is a keystone in human knowledge about statistical pattern recognition. For a dataset containing $C$ clusters, the classical solution to LDA extracts at most $C-1$ features. Here, we introduce a novel solution to classical LDA, called LDA++, that yields $C$ features, each interpretable as measuring similarity to one cluster. This novel solution bridges dimensionality reduction and multiclass classification. Specifically, we prove that,  for homoscedastic Gaussian data and under some mild conditions, the optimal weights of a linear multiclass classifier also make an optimal solution to LDA. In addition, we show that LDA++ reveals some important new facts about LDA that remarkably changes our understanding of classical multiclass LDA after 75 years of its introduction. We provide a complete numerical solution for LDA++ for the cases 1) when the scatter matrices can be constructed explicitly, 2) when constructing the scatter matrices is infeasible, and 3) the kernel extension.
\end{abstract}

\end{frontmatter}


\section{Introduction}\label{sec:introduction}
Linear discriminant analysis (LDA) is a fundamental statistical method for feature extraction and dimensionality reduction that has numerous applications in many scientific fields, including statistical pattern recognition, machine learning, and computer vision. 
This analysis was first proposed by Fisher \cite{fisher1936use} for two classes and, later, Rao \cite{rao1948utilization} generalized it to multiple classes. 
The essence of LDA is reducing the dimensionality of data using a linear transformation $Y=A^T X$ in such a way that the scatter of data belonging to the same classes is minimized while the scatter of data belonging to different classes is maximized.

The objective function of classical LDA is $tr\{(S_2^Y)^{-1}S_1^Y\}$, where $S_1^Y$ and $S_2^Y$ are two appropriately chosen scatter matrices in the feature space $Y$ among the total scatter matrix $S_t^Y$, the within-class scatter matrix $S_w^Y$, and the between-class scatter matrix $S_b^Y$. 
This objective function is invariant to nonsingular linear transformations in the feature space, which is a desirable property since these transformations do not affect the performance of the Bayes classifier \citep{fukunaga1990statistical}.
Consequently, LDA does not have a unique solution, as all nonsingular linear transformations of an optimal solution are themselves optimal solutions.
We consider the generalized objective function $tr\{(S_t^Y)^{\dagger}S_b^Y\}$ which has been advocated by Ye \citep{ye2005characterization} after investigating the objective functions for LDA comprehensively.
Traditionally, this objective function is optimized by solving the generalized eigenvalue problem
	\begin{eqnarray}
\label{eq:generalized-eigen}
S_1 A = S_2 A \Lambda,
\end{eqnarray}
where $S_1$ and $S_2$ refer to some scatter matrices in the input space $X$, $A$ is the matrix of eigenvectors, and $\Lambda$ is the diagonal matrix of eigenvalues. We refer to this method for solving LDA as EIG-LDA.

One of the main limitations of EIG-LDA is that the resulting linear feature extractors are not interpretable \citep{zhang2015sparse, wen2018robust, Dornaika2020141}. For example, in face recognition, Fisherfaces are a complex combination of training faces and don't provide a comprehensible interpretation of the functionality of LDA. 
In this paper, we propose a new numerical solution to classical LDA, which provides a prototype-based interpretation of LDA and its functionality.
We first speculate and then prove that $A=S_t^{\dagger}M$ is an optimal solution to LDA, where $S_t$ is the total scatter matrix and $M$ is the vector of means of all clusters minus the mean of all training data.
We refer to this novel solution of LDA as LDA++.

In contrast to EIG-LDA, where the number of extracted features is $C-1$, LDA++ extracts $C$ features, each one corresponding to a cluster.
The feature associated with each cluster shows the similarity of the input data to the center of that cluster, measured by the unsupervised metric induced by $S_t^\dagger$. 
Therefore, our method provides a prototype-based interpretation of the linear filters learned by LDA++. Interestingly, the requirement of having $C$ interpretable scores for each of the clusters has naturally occurred in the recent research of Zheng et al. \citep{Zheng2021}, where the authors, not having the machinery introduced in this paper, have used $C$ binary classification problems to learn specific feature spaces for each cluster.
Nevertheless, it may be argued that extracting one more feature is a weakness of the proposed method. 
We also prove that by throwing away any one of the features, the remaining $C-1$  features also optimize the objective function of LDA. 
We find the relation between EIG-LDA and  LDA++ and show that EIG-LDA can be regarded as the combination of LDA++ with $C-1$ features and a metric-changing transformation. 
We then apply a similar metric-changing transformation to the $C$-features LDA++ and call the resulting method EIG-LDA++.
In addition, we show that, under some mild conditions, $A=S_w^{-1}M$ is another optimal solution to LDA, and again the optimality is preserved even after one of the features is arbitrarily eliminated. 
This result connects LDA with multiclass classification since $A=S_w^{-1}M$ is also an optimal solution to multiclass classification of homoschedastic Gaussian data.

While our focus is on classical LDA, we try to modernize the proposed method in several aspects.
Firstly, as we stated above, we use the generalized objective function $tr\{(S_t^Y)^{\dagger}S_b^Y\}$ which is always defined.
Secondly, in classical LDA, the number of extracted features is limited to the number of classes minus one, which means a severe and inappropriate reduction in the dimensionality of data \citep{hastie1996discriminant, sugiyama2007dimensionality, lai2018robust}. 
The general remedy to this limitation is to use a clustering algorithm and incorporate the subclass structure into the LDA analysis \citep{zhu2006subclass, gkalelis2011mixture, wan2017separability, Chumachenko2021} and to choose all eigenvectors with nonzero eigenvalues for feature extraction\footnote{In fact, when less than $C-1$ features are extracted, LDA becomes non-optimal and may not yield the most discriminative features \citep{martinez2005linear, hamsici2008bayes}.}.
In this paper, we assume that LDA is applied to subclasses and $C$ stands for the number of clusters. 
Thirdly, for high-dimensional (HD) data, the computation and storage of the scatter matrices become prohibitive and, for small sample size (SSS) datasets, the empirical scatter matrices become very poor estimates of the true ones.
Some researchers proposed to circumvent this problem by first applying PCA/KPCA to reduce the dimensionality of data and then use LDA on top of the low-dimensional PCA/KPCA features \citep{belhumeur1997eigenfaces, yang2004essence, yang2005kpca}.
A more principled solution to this problem is a clever algorithm \citep{ye2005characterization} based on singular value decomposition (SVD), which we review in Section~\ref{sec:lda-num-alg}. 
We give two numerical algorithms for 1) the simple case of low-dimensional large sample size (LDLSS) setting, in which the scatter matrices can be computed explicitly, and 2) the challenging case of high-dimensional or small sample size (HD/SSS) setting, in which the solution is found without computing the scatter matrices.
Finally, we show how LDA++ leads to an efficient variant of kernel LDA which has many benefits over the classical one \citep{baudat2000generalized}.

The rest of the paper proceeds as follows.  
In Section~\ref{sec:lda-background}, we review some background material on LDA, introduce the classical numerical algorithms for LDLSS and HD/SSS settings, and review kernel LDA.
In Section~\ref{sec:linear-classification-background}, we briefly discuss the closely related problem of linear multiclass classification.
In Section~\ref{sec:obj}, we choose an objective function which is defined even for singular scatter matrices.
We introduce LDA++ in Section~\ref{sec:proposed-method} and propose two numerical algorithms for LDLSS and HD/SSS settings and a simpler numerical solution for kernel LDA.
In Section~\ref{sec:sep-dim-red-and-metric-changing}, we show that EIG-LDA and LDA++ solutions can be related with a metric-changing transformation and derive another solution called EIG-LDA++.
In Section~\ref{sec:new-findings-lda}, we mention some important new findings which are in contrast with general knowledge about LDA.
In Section~\ref{sec:experiments}, we experimentally evaluate the proposed method on some artificial and real-world datasets.
We conclude the paper in Section~\ref{sec:conclusion}.

\section{Linear Discriminant Analysis}\label{sec:lda-background}
In this section, we review some background material on LDA.
In Section~\ref{sec:notations}, we introduce the basic notations. 
We review the classical formulation of LDA in Section~\ref{sec:lda-classics} and give detailed numerical algorithms for LDA in LDLSS and HD/SSS settings in Section~\ref{sec:lda-num-alg}. 
We review classical kernel LDA \citep{baudat2000generalized} in Section~\ref{sec:lda-kernel}.
\subsection{Notations}
\label{sec:notations}
Assume that we have $N$ training samples $x_1,x_2,...,x_N\in\mathbb{R}^D$ which belong to $C$ clusters. 
In matrix notation, let $X$ be an $N\times D$ matrix whose rows are training samples.
For $c\in\{1,2,...,C\}$, let $s_c$ and $t_c$ denote the start and end indices for samples of cluster $c$.
Let $\mu$ be the mean of all training samples and, for $c\in\{1,2,...,C\}$, let $\mu_c$ be the mean of samples of cluster $c$.
We denote the identity matrices of all sizes with $I$, assuming that the size can be inferred from the context.
Let $e$ denote the $N$-dimensional vector of all ones.
For $n\in\{1,2,...,N\}$, we define $e^{<n>}$ as an $N$-dimensional zero vector with a one in the $n$-th entry.
For $c\in\{1,2,...,C\}$, let $e^{(c)}$ denote an $N$-dimensional vector in which the entries from $s_c$ to $t_c$ are one and the other entries are zero. 
We can write $\mu = \frac{1}{N} X^T e$ and, for $c\in\{1,2,...,C\}$, we have $\mu_c = \frac{1}{N_c} X^T e^{(c)}$, where $N_c$ denotes the number of samples of cluster $c$. The within-cluster ($S_w$), between-cluster ($S_b$), and total ($S_t$) scatter matrices are defined as
\begin{eqnarray}
\begin{aligned}
S_w &= \frac{1}{N} \sum_{c=1}^C \sum_{n=s_c}^{t_c} (x_n - \mu_c)(x_n - \mu_c)^T \\
S_b &= \sum_{c=1}^C \frac{N_c}{N} (\mu_c-\mu)(\mu_c-\mu)^T \\
S_t &= S_w + S_b = \frac{1}{N} \sum_{n=1}^N(x_n-\mu)(x_n-\mu)^T.
\end{aligned}
\end{eqnarray}
The scatter matrices can also be written as  \citep{howland2004generalizing, ye2005characterization} 
\begin{eqnarray}
\label{eq:SH-wbt}
\begin{aligned}
S_w &= X^T G_w X = H_w^T H_w \\
S_b &= X^T G_b X = H_b^T H_b \\
S_t &= X^T G_t X = H_t^T H_t,
\end{aligned}
\end{eqnarray}
where 
\begin{eqnarray}
\begin{aligned}
\label{eq:G-wbt}
G_w &= \frac{1}{N} \left(I - \sum_{c=1}^C \frac{1}{N_c} e^{(c)} {e^{(c)^T}}\right) \\
G_b &= \frac{1}{N} \left(\sum_{c=1}^C \frac{1}{N_c} e^{(c)} {e^{(c)^T}} - \frac{1}{N} e e^T\right) \\
G_t &= \frac{1}{N} \left(I- \frac{1}{N} ee^T\right)
\end{aligned}
\end{eqnarray}
and
\begin{eqnarray}
\label{eq:H-wbt}
\begin{aligned}
H_w &= \sqrt{\frac{1}{N}}\Big[x_{s_1}-\mu_1, ..., x_{t_1}-\mu_1,...,  
x_{s_C}-\mu_C, ..., x_{t_C}-\mu_C\Big]\\
H_b &= \begin{bmatrix} \sqrt{\frac{N_1}{N}}(\mu_1-\mu), ..., \sqrt{\frac{N_C}{N}}(\mu_C-\mu)\end{bmatrix}^T \\
H_t &= \frac{1}{\sqrt{N}} (X- e \mu^T). 
\end{aligned}
\end{eqnarray}
\subsection{Problem Formulation}
\label{sec:lda-classics}
LDA seeks for a matrix $A_{D\times F}$ such that the transformation $y=A^T x$ maps an input data $x\in\mathbb{R}^D$ into a feature vector $y\in \mathbb{R}^F$ with fewer dimensions \citep{fukunaga1990statistical}. 
The scatter matrices in the feature space are related to those of the input space with the following relations:
\begin{eqnarray*}
\begin{aligned}
S_w^Y &= A^T S_w A \\
S_b^Y &= A^T S_b A \\
S_t^Y &= A^T S_t A .
\end{aligned}
\end{eqnarray*}

LDA chooses a transformation $A$ that maximizes the objective function $tr\left\{(S_2^Y)^{-1}S_1^Y\right\}$. 
All choices $(S_b^Y,S_w^Y), (S_b^Y,S_t^Y)$, and $(S_t^Y,S_w^Y)$  for the pair $(S_1^Y,S_2^Y)$ yield the same result \citep{fukunaga1990statistical, howland2004generalizing}\footnote{
The above ratio trace objecive function is equivalent to the trace ratio objective function $tr\{S_1^Y\}/tr\{S_2^Y\}$ (see footnote 5 in \cite{wan2017separability}). The arguments mentioned by Wang et al. \citep{wang2007trace} and Jia et al. \citep{jia2009trace} about the inferiority of the ratio trace objective function are only relevant to those dimensionality reduction methods that impose some form of restriction on the transformation matrix $A$, not LDA.}.
Since the rank of $S_b$ is less than $C$, at most $C-1$ features can be extracted (i.e. $F\le C-1$).
The classical solution to this problem is to solve the generalized eigenvalue problem
\begin{eqnarray*}
S_1 \phi_i = \lambda_i S_2 \phi_i
\end{eqnarray*}
and to choose $A=[\phi_1,...,\phi_F]$, where $\phi_i$ is the eigenvector corresponding to the $i$-th largest non-zero eigenvalue $\lambda_i$.
\subsection{Classical numerical algorithms for LDA}
\label{sec:lda-num-alg}
In this section, we review classical algorithms for solving LDA.
We seek for a matrix solution $A$ to the problem of maximizing $tr\left\{(A^T S_2 A)^{\dagger}A^T S_1 A\right\}$.
Classical methods for LDA find a solution by solving the generalized eigenvalue problem
 (\ref{eq:generalized-eigen}).
In the LDLSS setting, (\ref{eq:generalized-eigen}) is solved by computing the scatter matrices explicitly and then computing the eigenvectors corresponding to the non-zero eigenvalues. 
The pseudocode for this algorithm is shown in Algorithm~\ref{alg:lda-eigen}.
However, in the HD/SSS setting, storing the scatter matrices requires large amount of memory and/or the scatter matrices become singular.
In these situations, Algorithm~\ref{alg:lda-svd} which is based on SVD can efficiently find the solution.
Ye \citep{ye2005characterization} proved that this algorithm solves (\ref{eq:generalized-eigen}) when $S_2=S_t$ and $S_1=S_b$\footnote{ 
Nevertheless, Algorithm~\ref{alg:lda-svd} is erroneously used in Scikit-learn package with $S_2=S_w$.
A counterexample, showing that Algorithm~\ref{alg:lda-svd} actually does not solve (\ref{eq:generalized-eigen}) when $S_2=S_w$, is given in 
the supplementary material in the jupyter notebook file 'exp\_verify\_svd\_solver.ipynb'.
In addition, recently Cao et al. \citep{Cao2019218}, not being aware of \citep{ye2005characterization}, rediscovered the importance of using pseudoinverse in the objective function of LDA but proposed the  problematic choice of $S_2=S_w$ and $S_1=S_b$.
}. 
For the sake of integrity, here, we introduce a new mathematical derivation of this algorithm.
We proceed with general $S_1$ and $S_2$ and meanwhile show exactly where the assumptions $S_2=S_t$ and $S_1=S_b$ become necessary.
Assuming that $S_2 = H_2^T H_2$ and $S_1 = H_1^T H_1$, we rewrite (\ref{eq:generalized-eigen}) as
\begin{eqnarray}
H_1^T H_1 A = H_2^T H_2 A \Lambda.
\label{eq:H-eigensystem}
\end{eqnarray}
Let $U \begin{bmatrix}\Sigma & 0 \\ 0 & 0 \end{bmatrix} V^T$ be the full SVD of $H_2$, where $\Sigma$ is a diagonal matrix with non-zero diagonal elements and $U$ and $V$ are unitary matrices.
Assuming that $H_2$ has $D_2$ rows, the dimensions of $U$ and $V$ are $D_2\times D_2$ and $D\times D$, respectively. 
Let $r$ denote the dimension of the square matrix $\Sigma$. 
Then, the eigen-decomposition of $S_2$ is 
$V \begin{bmatrix}\Sigma^2 & 0 \\ 0 & 0 \end{bmatrix} V^T$. 
Since $V$ is square and unitary, $V^T V = V V^T = I$. Swapping the two sides of (\ref{eq:H-eigensystem}), we have
\begin{eqnarray*}
\begin{aligned}
V \begin{bmatrix}\Sigma^2 & 0 \\ 0 & 0 \end{bmatrix} V^T A \Lambda &= H_1^T H_1 A
\Leftrightarrow \\
 \begin{bmatrix}\Sigma^2 & 0 \\ 0 & 0 \end{bmatrix} V^T A \Lambda &= V^T H_1^T H_1 V V^T A,
\end{aligned}
\end{eqnarray*} 
and therefore  
\begin{eqnarray}
\label{eq:lda-svd-general}
\begin{aligned}
\begin{bmatrix}\Sigma^2 & 0 \\ 0 & 0 \end{bmatrix} V^T A \Lambda
&= \begin{bmatrix}\Sigma & 0 \\ 0 & I \end{bmatrix} 
\begin{bmatrix}\Sigma^{-1} & 0 \\ 0 & I \end{bmatrix} 
V^T H_1^T H_1 V V^T A.
\end{aligned}
\end{eqnarray}

Up until now, the derivation is true for all choices for the scatter matrices $S1$ and $S2$. 
Now, we ought to be specific about the choice of these scatter matrices. 
Ye \citep{ye2005characterization} proved that, when $S_2=S_t$ and $S_1=S_b$, then the last $D-r$ rows and columns of $V^T S_b V=V^T H_1^T H_1 V$ are zero \citep[see][Section 3.1, Eq.(12)]{ye2005characterization}, i.e.
\begin{eqnarray*}
V^T S_b V=V^T H_1^T H_1 V
= \begin{bmatrix} \tilde{S}_b & 0 \\ 0 & 0 \end{bmatrix}.
\end{eqnarray*}
It follows that the last $D-r$ rows on both sides of (\ref{eq:lda-svd-general}) are zero.
Consequently, (\ref{eq:lda-svd-general}) can be continued as follows:
\begin{eqnarray*}
\begin{aligned}
& \begin{bmatrix}\Sigma^2 & 0 \\ 0 & 0 \end{bmatrix} V^T A \Lambda
= \begin{bmatrix}\Sigma & 0 \\ 0 & I \end{bmatrix} 
\begin{bmatrix}\Sigma^{-1} & 0 \\ 0 & I \end{bmatrix} 
V^T H_1^T H_1 V V^T A \Leftrightarrow\\
& \begin{bmatrix}\Sigma^2 & 0\end{bmatrix} V^T A \Lambda
= \begin{bmatrix}\Sigma & 0 \end{bmatrix} 
\begin{bmatrix}\Sigma^{-1} & 0 \\ 0 & I \end{bmatrix} 
V^T H_1^T H_1 V V^T A \Leftrightarrow\\
&  \begin{bmatrix}\Sigma & 0\end{bmatrix} 
\begin{bmatrix}\Sigma & 0 \\ 0 & I \end{bmatrix} V^T A \Lambda \\
&= 
\begin{bmatrix}\Sigma & 0 \end{bmatrix} 
\begin{bmatrix}\Sigma^{-1} & 0 \\ 0 & I \end{bmatrix} 
V^T H_1^T H_1 V 
\begin{bmatrix}\Sigma^{-1} & 0 \\ 0 & I \end{bmatrix} 
\begin{bmatrix}\Sigma & 0 \\ 0 & I \end{bmatrix} V^T A.
\end{aligned}
\end{eqnarray*}
Therefore
\begin{eqnarray}
\begin{bmatrix}\Sigma & 0 \end{bmatrix} B \Lambda= 
\begin{bmatrix}\Sigma & 0 \end{bmatrix} Y^T Y B,
\label{eq:eigensystem-B}
\end{eqnarray}
where 
\begin{eqnarray}
\label{eq:B-definition}
B=\begin{bmatrix}\Sigma & 0 \\ 0 & I \end{bmatrix} V^T A
\end{eqnarray} 
and
\begin{eqnarray}
\label{eq:Y-definition}
Y=H_1 V \begin{bmatrix}\Sigma^{-1} & 0 \\ 0 & I \end{bmatrix}.
\end{eqnarray}

Let $\tilde{U} \tilde{\Sigma} \tilde{V}^T$ be the reduced SVD of $Y$. 
By substituting $Y^T Y$ with $\tilde{V} \tilde{\Sigma}^2 \tilde{V}^T$ in 
 (\ref{eq:eigensystem-B}), we arrive at the equation
\begin{eqnarray}
\begin{aligned}
\label{eq:svd-eigendecomp}
\begin{bmatrix}\Sigma & 0 \end{bmatrix} B \Lambda= 
\begin{bmatrix}
\Sigma & 0 \end{bmatrix} \tilde{V} \tilde{\Sigma}^2 \tilde{V}^T B.
\end{aligned}
\end{eqnarray}
It is now easy to verify that the choices $B=\tilde{V}$ and $\Lambda = \tilde{\Sigma}^{2}$ solve (\ref{eq:svd-eigendecomp}). 
Restating the solution in terms of $A$, the solution of (\ref{eq:H-eigensystem}) is $A=V\begin{bmatrix}\Sigma^{-1} & 0 \\ 0 & I \end{bmatrix} \tilde{V}$.

There is still a difference between our derived solution and the solution given by Algorithm~\ref{alg:lda-svd}. Here, we have used full SVD while Algorithm~\ref{alg:lda-svd} uses reduced SVD. 
From (\ref{eq:Y-definition}), we have the following equation which shows that the last $D-r$ rows and columns of $Y^T Y$ are zero:
\begin{eqnarray}
\label{eq:YTY-1}
\begin{aligned}
Y^T Y &= 
\begin{bmatrix}\Sigma^{-1} & 0 \\ 0 & I \end{bmatrix} 
V^T H_1^T H_1 V 
\begin{bmatrix}\Sigma^{-1} & 0 \\ 0 & I \end{bmatrix} 
=
\begin{bmatrix} \Sigma^{-1}\tilde{S}_b \Sigma^{-1} & 0 \\ 0 & 0 \end{bmatrix}. 
\end{aligned}
\end{eqnarray}
On the other hand, from $Y=\tilde{U} \tilde{\Sigma} \tilde{V}^T$ we have
\begin{equation}
\label{eq:YTY-2}
Y^T Y=\tilde{V} \tilde{\Sigma}^2 \tilde{V}^T
= (\tilde{\Sigma} \tilde{V}^T)^T
\tilde{\Sigma} \tilde{V}^T.
\end{equation} 
The above equation shows that the $i$-th diagonal entry of $Y^T Y$ is equal to the squared norm of the $i$-th column of $\tilde{\Sigma} \tilde{V}^T$. 
Considering the equality of (\ref{eq:YTY-1}) and (\ref{eq:YTY-2}), it follows that the last $D-r$ columns of $\tilde{\Sigma} \tilde{V}^T$ are zero. 
Since $\tilde{\Sigma}$ is non-singular, it follows that the last $D-r$ columns of $\tilde{V}^T$, or equivalently the last $D-r$ rows of $\tilde{V}$ are zero.
Thus, we can write $\tilde{V}=\begin{bmatrix}\tilde{V}_1 \\ 0 \end{bmatrix}$, where the number of rows of $\tilde{V}_1$ is $r$.

Now, Let $V=[V_1|V_2]$ be the partitioning of the columns of the matrix $V$ into the first $r$ and the remaining $D-r$ columns. The reduced SVD of $H_2$ yields only $V_1$. We have 
\begin{eqnarray*}
\begin{aligned}
A &= V\begin{bmatrix}\Sigma^{-1} & 0 \\ 0 & I \end{bmatrix} \tilde{V}
  = \begin{bmatrix}V_1 & V_2\end{bmatrix} 
\begin{bmatrix}\Sigma^{-1} & 0 \\ 0 & I \end{bmatrix} \begin{bmatrix}\tilde{V}_1 \\ 0\end{bmatrix}  
  =
V_1\Sigma^{-1}\tilde{V_1}.
\end{aligned}
\end{eqnarray*}
This justifies the use of reduced SVD twice in Algorithm~\ref{alg:lda-svd}.

\begin{algorithm}
\caption{Classical LDA algorithm in LDLSS setting}
\label{alg:lda-eigen}
\begin{algorithmic}[1]
\STATE Compute $S_1$ and $S_2$ explicitly.
\STATE Solve the eigensystem $S_1 A = S_2 A \Lambda$.
\RETURN A
\end{algorithmic}
\end{algorithm}

\begin{algorithm}
\caption{Classical LDA algorithm in HD/SSS setting}
\label{alg:lda-svd}
\begin{algorithmic}[1]
\STATE Compute $H_1$ and $H_2$ using (\ref{eq:H-wbt}).
\STATE Compute the reduced SVD of $H_2$ to obtain $H_2 = U \Sigma V^T$.
\COMMENT {Eigen-decomposition of $S_2$ is $S_2 = V \Sigma^2 V^T$.}
\STATE Compute $Y = H_1 V \Sigma^{-1}$.
\STATE Compute the reduced SVD of $Y$ to obtain $Y = \tilde{U} \tilde{\Sigma} \tilde{V}^T$.
\STATE $A = V \Sigma^{-1}\tilde{V}$
\RETURN A
\end{algorithmic}
\end{algorithm}
\subsection{Kernel LDA}
\label{sec:lda-kernel}
In this section, we review the method of Baudat and Anouar \cite{baudat2000generalized} for generalizing LDA to the feature space of a kernel function.
Assume that the input data belong to a set $X$.
Let $k:X\times X \to \mathbb{R}$ be a positive definite kernel function and let $\Phi:X\to H$ be a kernel map, where $H$ is a Hilbert space associated with the kernel function $k$.
The feature map $\Phi$ has the key property that for all $x,z\in X$, the equality $k(x,z)=\langle \Phi(x), \Phi(z) \rangle _H$ holds, where $\langle \rangle _H$ denotes the inner product in the Hilbert space $H$.
Assuming that the training data are centered in $H$, Baudat and Anouar \cite{baudat2000generalized} showed that the equivalent form of the eigenproblem $S_t A = S_b A \Lambda$ in the feature space is
\begin{eqnarray}
\label{eq:lda-kernel-problem}
KWK\alpha = \lambda KK\alpha,
\end{eqnarray}
where K is the kernel matrix of the training data, 
$\lambda$ is an eigenvalue, $\alpha$ is the vector of coefficients of the expansion of an eigenvector in the feature space $H$, as defined by
\begin{eqnarray*}
\sum_{n=1}^N\alpha_{n} \Phi(x_{n}),
\end{eqnarray*}
and $W$ is a block diagonal matrix containing $C$ blocks where the $i$-th block is an $N_i\times N_i$ matrix with all entries equal to $\frac{1}{N_i}$ as shown below:
 
\begin{eqnarray*}
W=\begin{bmatrix} \frac{1}{N_1} & ... & \frac{1}{N_1} &   &   &    &   &  &  &   \\
				  \vdots        &     & \vdots        &   &   &    &      &  &  &   \\
				  \frac{1}{N_1} & ... & \frac{1}{N_1} &   &   &    &  &   &  &   \\
				    &  &   & \frac{1}{N_2} & ... & \frac{1}{N_2} &    &  &   \\
				    &  &   & \vdots &  & \vdots &  &   &   &   \\
				    &  &   & \frac{1}{N_2} & ... & \frac{1}{N_2} &  &   &   &   \\				  
				    &  &   &   &  &   &  \ddots  &    &   &   \\				  
				    &  &   &   &  &   &  & \frac{1}{N_C} & ... & \frac{1}{N_C}\\
				    &  &   &   &  &   &  & \vdots &  & \vdots \\
				    &  &   &   &  &   &  & \frac{1}{N_C} & ... & \frac{1}{N_C}
\end{bmatrix}.
\end{eqnarray*}

To solve (\ref{eq:lda-kernel-problem}), Baudat and Anouar \cite{baudat2000generalized} proposed to first find the reduced eigenvector decomposition of the matrix $K$ as $U \Gamma U^T$. Then, (\ref{eq:lda-kernel-problem}) can be rewritten as
\begin{eqnarray}
\label{eq:kfda-big-formula}
U \Gamma U^T W U \Gamma U^T \alpha = \lambda U \Gamma U^T U \Gamma U^T \alpha.
\end{eqnarray}
Defining $\beta = \Gamma U^T \alpha$, (\ref{eq:kfda-big-formula}) can be rewritten as
\begin{eqnarray*}
U \Gamma U^T W U \beta = \lambda U \Gamma \beta.
\end{eqnarray*}
They assumed that this equation can be simplified to
\begin{eqnarray}
\label{eq:kfda-eigensystem}
U^T W U \beta = \lambda \beta
\end{eqnarray}
and proposed solving (\ref{eq:kfda-eigensystem}) for $\beta$ and $\lambda$ and then obtaining $\alpha$ by $U \Gamma^{\dagger} \beta$.
Algorithm~\ref{alg:lda-kernel} shows the classical method of Baudat and Anouar \cite{baudat2000generalized} for training kernel LDA.

\begin{algorithm}
\caption{Classical algorithm for training kernel LDA}
\label{alg:lda-kernel}
\begin{algorithmic}[1]
\STATE {Compute the kernel matrix $K$ of training data.}
\STATE {Compute the eigen-decomposition $K=U\Gamma U^T$.}
\STATE {Solve the eigensystem (\ref{eq:kfda-eigensystem}) for $\beta$.}
\FOR {$i=1,2,...C-1$}
\STATE {$\alpha^{(i)} = U \Gamma^{\dagger} \beta^{(i)}$}
\STATE {Divide $\alpha^{(i)}$ by $\sqrt{{\alpha^{(i)^T} K \alpha^{(i)}}}$ for normalization.}
\ENDFOR
\RETURN {$\alpha^{(1)},...,\alpha^{(C-1)}$}
\end{algorithmic}
\end{algorithm}

\section{Linear Multiclass Classification}\label{sec:linear-classification-background}
A problem that is highly related to LDA is the problem of linear multiclass classification. This relation is so strong that Rao \cite{rao1948utilization} analyzed linear discriminant analysis and linear multiclass classification in the same paper. He proposed the use of $C$ functions, which he called linear discriminant scores, for multiclass classification. 
In this section, we first find the Bayes optimal classifier for homoscedastic Gaussian data.
Then, we review some recent work on interpretable multi-prototype multiclass classification.
The topics of this section provide a background for investigating relations between multiclass classification and LDA in subsequent parts of the paper.
\subsection{Bayes linear classifiers for homoscedastic Gaussians}
\label{sec:bayes-linear-classifier}
Assume that data come from a homoscedastic Gaussian distribution, i.e. the covariance matrices of all classes are the same.
Assume that $\mu_1,...,\mu_C$ are the means of the classes and that $\Sigma_w$ is the common covariance matrix.
Note that since this covariance matrix shows the scatter of data within a class, it is the true within-class scatter matrix, and thus we have denoted it with a subscript $w$.
Besides, assume that, for $c\in\{1,...,C\}$, $P(c)$ denotes the prior probability of each class.
According to the Bayes theorem, the optimal classification rule is 
\begin{eqnarray*}
\begin{aligned}
\argmax_{c\in\{1,...,C\}}{P(c|x)} &= \argmax_{c\in\{1,...,C\}}{P(x|c)P(c)} \\
&=\argmax_{c\in\{1,...,C\}}{log{P(x|c)}+log{P(c)}}.
\end{aligned}
\end{eqnarray*}
By assumption, data of each class follows a normal distribution $P(x|c)=\mathcal{N}(x;\mu_c,\Sigma_w)$.
Therefore, the optimal Bayes classifier is
\begin{eqnarray*}
\begin{aligned}
&\argmax_{c\in\{1,...,C\}}{-\frac{1}{2}(x-\mu_c)^T \Sigma_w^{-1}(x-\mu_c) + log{P(c)}}\\
=& \argmax_{c\in\{1,...,C\}}{-\frac{1}{2} x^T \Sigma_w^{-1}x + (\Sigma_w^{-1}\mu_c)^T x -\frac{1}{2} \mu_c^T \Sigma_w^{-1}\mu_c + log{P(c)}}\\
=& \argmax_{c\in\{1,...,C\}}{w_c^T x+b_c},
\end{aligned}
\end{eqnarray*}
where $w_c = \Sigma_w^{-1}\mu_c$ and $b_c = -\frac{1}{2}\mu_c^T \Sigma_w^{-1}\mu_c + log {P(c)}$.
For each $c\in\{1,...,C\}$ we have
\begin{eqnarray*}
\begin{aligned}
P(c)&=1-\sum_{c'\neq c}{P(c')}\\
\mu_c &= \frac{1}{N_c}\left(\sum_{n=1}^{N}X_n - \sum_{c'\neq c} {N_{c'} \mu_{c'}}\right).
\end{aligned}
\end{eqnarray*}
Therefore, one of the $C$ linear discriminator scores is redundant and can be computed from the rest.
If we neglect the bias term, then the best discriminating features are $y=A^T x$, where $A=\Sigma_w^{-1}[\mu_1,...,\mu_{C-1}]$.
For $c\in\{1,...,C\}$, Rao \cite{rao1948utilization} called the values
\begin{eqnarray*}
\begin{aligned}
L_c = (\Sigma_w^{-1}\mu_c)^T x -\frac{1}{2} \mu_c^T \Sigma_w^{-1}\mu_c,
\end{aligned}
\end{eqnarray*}
which are independent of any \textit{a priori} information, linear discriminator scores. 
In Section~\ref{sec:relation-with-classification}, we will show that this optimal solution to the classification problem is also an optimal solution to the classical LDA.

\subsection{Multi-prototype Multiclass classification}
Until recently, linear multiclass classification methods only assigned one unit to each class. Ghiasi-Shirazi \cite{ghiasi2019competitive} proposed the competitive cross-entropy (CCE) method for training single-layer neural networks with multiple neurons for each class. He showed that the neurons of each class specialize at recognizing a cluster of data for that class. Considering each neuron as a feature extractor, a single-layer neural network with $C$ output neurons, one for each cluster, is similar to LDA++, as both methods extract $C$ features. 
In a more recent paper, Zarei-Sabzevar et al. \citep{ZareiGhiasi2022} proposed the $\pm$ED-WTA network which is a single-layer winner-takes-all neural network that yields prototypes for each of the $C$ output neurons. 
We will visualize the feature extractors and prototypes learned by LDA++ and $\pm$ED-WTA in the experiments of Section~\ref{sec:experiments-mnist}.

\section{Objective function for multiclass LDA}
\label{sec:obj}
There are several objective functions for multiclass LDA \citep{fukunaga1990statistical}.
These objective functions are defined based on two scatter matrices $S_1$ and $S_2$, which are selected from $S_w$, $S_b$, and $S_t$.
Optimizing these objective functions is equivalent to solving the generalized eigenvalue problem (\ref{eq:generalized-eigen}). In this paper, we follow the direction of Fukunaga \cite{fukunaga1990statistical}, Howland and Park \cite{howland2004generalizing}, Ye \cite{ye2005characterization} and consider the objective function $tr\left\{(A^T S_2 A)^{-1} A^T S_1 A\right\}$ which is invariant to nonsingular transformations of the matrix $A$.
Howland and Park \cite{howland2004generalizing} generalized this to the case where $S_2$ is singular and formulated the problem as a generalized singular value decomposition. 
They proposed to solve the following even more generalized eigenvalue problem
\begin{eqnarray}
\label{eq:S1ALam1eqS2ALam2}
S_1 A \Lambda_1 = S_2 A \Lambda_2,
\end{eqnarray}
where $\Lambda_1^2 + \Lambda_2^2 = I$ and $\Lambda_1$ and $\Lambda_2$ are diagonal matrices.
The value of the objective function is then $tr\left\{\Lambda_1^{-2} \Lambda_2^2\right\}$ which becomes infinity when at least one of the diagonal elements of $\Lambda_1$ is zero.
This can happen only if a generalized eigenvector (i.e. a column of $A$) falls within $\mathcal{N}(S_2)$ but not $\mathcal{N}(S_1)$, where $\mathcal{N}$ denotes the null space.
When a generalized eigenvector falls within the null space of both $S_1$ and $S_2$, then one column in both sides of the equation becomes zero and we have an option in choosing the corresponding entries in 
$\Lambda_1$ and $\Lambda_2$.
In this situation, Howland and Park \cite{howland2004generalizing} proposed to set the entry in $\Lambda_2$ equal to zero, so that its corresponding eigenvectors would be excluded.
In this paper, we choose $S_2=S_t$ since this choice leads to a more robust formulation of LDA compared to $S_2=S_w$ \citep{lu2005regularization}. 
We now prove a useful lemma.
\begin{lemma}
\label{lemma:nullspace-st-sb}
$\mathcal{N}(S_t)\subseteq \mathcal{N}(S_b)$.
\end{lemma}
\begin{proof}
From (\ref{eq:SH-wbt}) and (\ref{eq:G-wbt}) we have
\begin{eqnarray*}
S_t = \frac{1}{N} X^T (I - \frac{1}{N} e e^T) X.
\end{eqnarray*}
Assume that $v\in \mathcal{N}(S_t)$. Then,
\begin{eqnarray*}
S_t v= \frac{1}{N} X^T (I - \frac{1}{N} e e^T) X v = 0.
\end{eqnarray*}
Since $(I - \frac{1}{N} e e^T)$ is a projection matrix, as one can easily show, we have
\begin{eqnarray*}
\begin{aligned}
v^T S_t v &= \frac{1}{N} v^T X^T (I - \frac{1}{N} e e^T)(I - \frac{1}{N} e e^T) X v \\
&= \left\| (I - \frac{1}{N} e e^T) X v\right\|^2=0.
\end{aligned}
\end{eqnarray*}
Therefore, $(I - \frac{1}{N} e e^T) X v=0$. Defining $w=Xv$, we have
\begin{eqnarray*}
\begin{aligned}
w = \frac{1}{N} e e^T w = (\frac{1}{N} e^T w) e = \alpha e,
\end{aligned}
\end{eqnarray*}
where $\alpha = \frac{1}{N} e^T w$ is a scalar. This shows that $w$ is a vector whose elements are all equal. Now, we show that $v\in \mathcal{N}(S_b)$. From (\ref{eq:SH-wbt}) and (\ref{eq:G-wbt}) we have

\begin{eqnarray*}
\begin{aligned}
S_b v &= X^T \frac{1}{N} \left(\sum_{c=1}^C \frac{1}{N_c} e^{(c)} e^{(c)^T} - \frac{1}{N} e e^T\right) X v \\
&= X^T \frac{1}{N} \left(\sum_{c=1}^C \frac{1}{N_c} e^{(c)} {e^{(c)^T}} - \frac{1}{N} e e^T\right) w \\
&= X^T \frac{\alpha}{N} \left(\sum_{c=1}^C \frac{1}{N_c} e^{(c)} e^{(c)^T}e - \frac{1}{N} e e^T e\right) \\
&= X^T \frac{\alpha}{N} \left(\sum_{c=1}^C \frac{1}{N_c} e^{(c)} N_c - \frac{N}{N}e\right) \\
&= X^T \frac{\alpha}{N} \left(\sum_{c=1}^C e^{(c)} - e\right) 
= X^T \frac{\alpha}{N} \left(e - e\right)= 0.
\end{aligned}
\end{eqnarray*}
This completes the proof.
\end{proof}

The following proposition shows that for $S_2=S_t$ and $S_1=S_b$ the objective function is always finite.

\begin{proposition}
\label{proposition:finiteness-of-obj}
For the choices $S_2 = S_t$ and $S_1=S_b$, the matrix $\Lambda_1$ in (\ref{eq:S1ALam1eqS2ALam2}) does not contain any zero diagonal entries.
\end{proposition}
\begin{proof}
Assume on the contrary that in (\ref{eq:S1ALam1eqS2ALam2}), a diagonal entry $\lambda_1$ of $\Lambda_1$ is zero and its associated value in $\Lambda_2$ is $\lambda_2$ and its associated generalized eigenvector in $A$ is $a$.
Since $\lambda_1=0$ and $\lambda_1^2+\lambda_2^2=1$, the value of $\lambda_2$ is $1$. From (\ref{eq:S1ALam1eqS2ALam2}) we have
\begin{eqnarray*}
\lambda_1 S_b a = \lambda_2 S_t a = 0.
\end{eqnarray*}
Therefore, $a\in \mathcal{N}(S_t)$ and by lemma~\ref{lemma:nullspace-st-sb}, it follows that $a\in \mathcal{N}(S_b)$.
Consequently, as proposed by Howland and Park \cite{howland2004generalizing}, since $a$ is in the null space of both $S_b$ and $S_t$, we should have chosen $\lambda_1=1$ and $\lambda_2=0$ to exclude this useless vector from the solution.
Therefore, the original assumption that an entry of $\Lambda_1$ is zero is false and the proposition is proved.
\end{proof}

It can be shown that Proposition~\ref{proposition:finiteness-of-obj} does not hold for $S_2=S_w$.
Since we are ensured that for $S_1=S_b$ and $S_2=S_t$, the matrix $\Lambda_1$ is invertible, we can rewrite (\ref{eq:S1ALam1eqS2ALam2}) as 
\begin{eqnarray*}
S_1 A = S_2 A (\Lambda_2\Lambda_1^{-1}) = S_2 A \Lambda,
\end{eqnarray*}
which reduces (\ref{eq:S1ALam1eqS2ALam2}) to (\ref{eq:generalized-eigen}).
In the HD/SSS setting, the scatter matrix $S_t$ in the input space is singular. 
However, even in this case the matrix $A^T S_t A$, which is the total scatter matrix in the feature space, is almost always non-singular.
The reason is that, since each cluster contains at least one sample, the number of extracted features (which is at most $C-1$) cannot be more than the number of samples. 
In other words, viewed in the feature space, all problems are of LDLSS type.
In our experiments, we never observed even a single case in which the matrix $S_t$ in the feature space had become singular.
However, the criterion $tr\{(A^T S_t A)^{-1} A^T S_b A\}$ has the drawback that it is not defined when the matrix $A$ is badly chosen and the scatter matrix $A^T S_t A$ is singular.
To solve this problem, in this paper, we define the objective function $J(A)=tr\left\{(A^T S_t A)^{\dagger} A^T S_b A\right\}$. 
Interestingly, this criterion had been introduced and chosen by Ye \cite{ye2005characterization} as the preferred objective function for HD/SSS setting.

What we want to add here is that although the use of pseudoinverse is necessary from a theoretical point of view, however, the total-scatter matrix formed in the feature space is non-singular and the criticism made by Juefei-Xu and Savvides \cite{juefei2016multi}, that the use of pseudoinverse leads to approximate solutions,  is not true.
On the other hand, as there is a possibility that $S_w$ becomes singular in the feature space, using the pseudoinverse of $S_w$, as was proposed by Cao et al.\cite{Cao2019218}, may actually change the objective function
$tr\left\{(A^T S_t A)^{-1} A^T S_b A\right\}$.

\section{Proposed Method: LDA++}\label{sec:proposed-method}
In this section, we introduce our novel method for solving LDA, which we call LDA++.
First, in Section~\ref{sec:prop-speculating-solution}, we speculate a novel solution to LDA which does not involve solving an eigensystem.  
Then, in Section~\ref{sec:fisher-opt}, in a separate theorem, we prove that the speculated solution actually optimizes the objective function of LDA. 
In Section~\ref{sec:fisher-opt}, we proceed to prove, in another theorem, that by throwing away any one of the features, the remaining $C-1$ features also optimize the  objective function of LDA.  
In Section~\ref{sec:prop-lda-num-alg}, in parallel to subsection~\ref{sec:lda-num-alg}, we propose two numerical algorithms for LDLSS and HD/SSS settings.
Finally, we propose a simpler numerical solution to kernel LDA in Section~\ref{sec:prop-lda-kernel}. 
\subsection{Speculating a solution}
\label{sec:prop-speculating-solution}
In Section~\ref{sec:lda-kernel}, we stated that Baudat and Anouar \cite{baudat2000generalized} finally arrived at the eigensystem (\ref{eq:lda-kernel-problem}).
In this section, we assume that the kernel function $k(x,z)=(x-\mu)^T (z-\mu)$ is used so that the analysis comes back to the input space and the kernel matrix becomes $K=(I-\frac{1}{N}ee^T)XX^T(I-\frac{1}{N}ee^T)$. 
The subtraction of mean $\mu$ is necessary since (\ref{eq:lda-kernel-problem}) is derived under the assumption that data are centered in the feature space of the kernel function.
For this kernel function, we have the explicit feature map $\Phi(x)=x-\mu$.
Defining $\beta = K\alpha$, and eliminating the leading $K$ from both sides of (\ref{eq:lda-kernel-problem}), we arrive at 
\begin{eqnarray}
\label{eq:eigval-one}
W\beta = \lambda \beta.
\end{eqnarray}

We notice that $W$ is a projection matrix and consequently all of its eigenvalues are either $0$ or $1$. 
Without any computation, unnormalized eigenvectors of $W$ with eigenvalue $1$ can be readily figured out.
The eigenvectors of $W$ with eigenvalue $1$, i.e. the solutions $\beta$ of (\ref{eq:eigval-one}), are $e^{(1)},...,e^{(C)}$. We now solve the equation $\beta = K\alpha$ for $\alpha$. For $\beta=e^{(c)}$ we obtain
$e^{(c)} = (I-\frac{1}{N}ee^T)XX^T(I-\frac{1}{N}ee^T) \alpha^{(c)}$, where 
$c\in\left\{1,...,C\right\}$.
Since we are only speculating a solution, without mathematical rigor, we project the two sides of the equation on the centered training samples and proceed as follows:

\begin{eqnarray}
\label{eq:kernel-expansion-reslut}
\begin{aligned}
&X^T(I-\frac{1}{N}ee^T) e^{(c)} = X^T(I-\frac{1}{N}ee^T) XX^T(I-\frac{1}{N}ee^T) \alpha^{(c)} \\
&\Rightarrow  N_c (\mu_c - \mu) = N S_t X^T(I-\frac{1}{N}ee^T) \alpha^{(c)} \\
&\Rightarrow  \frac{N_c}{N} (\mu_c - \mu) = S_t (X^T-\mu e^T) \alpha^{(c)} \\
&\Rightarrow  (X^T-\mu e^T) \alpha^{(c)} = \frac{N_c}{N} S_t^\dagger (\mu_c - \mu),
\end{aligned}
\end{eqnarray}
where we have used the fact that $(I-\frac{1}{N}ee^T)$ is a projection matrix and therefore it is equal to its square.
The eigenvector $\phi_c$ in the feature space is
\begin{eqnarray*}
\begin{aligned}
\phi_c&=\sum_{n=1}^N \alpha^{(c)}_n \phi(x_n) = \sum_{n=1}^N \alpha^{(c)}_n (x_n-\mu) \\
&=(X^T-\mu e^T) \alpha^{(c)} = \frac{N_c}{N} S_t^\dagger (\mu_c - \mu),
\end{aligned}
\end{eqnarray*}
where the last equality was obtained in (\ref{eq:kernel-expansion-reslut}).
Therefore, we guess that one matrix $A$ that maximizes the objective function $tr\left\{(A^T S_t A)^{\dagger} A^T S_b A\right\}$ is 
\begin{eqnarray*}
\begin{aligned}
A &=[\phi_1,\phi_2,...,\phi_C] = S_t^\dagger M P,
\end{aligned}
\end{eqnarray*}
where $M=[\mu_1-\mu,...,\mu_C-\mu]$ and $P$ is a diagonal matrix with entries $\frac{N_1}{N},...,\frac{N_C}{N}$.
In contrast to EIG-LDA which consists of $C-1$ vectors, this solution contains $C$ vectors, each corresponding to a cluster.

\subsection{Proving the optimality of the solution}
\label{sec:fisher-opt}
In this section, we prove that the guess of the previous section is true. First, we note that the objective function is invariant to nonsingular transformations in the feature space \citep{fukunaga1990statistical}. Since $P$ is nonsingular, it suffices to prove the optimality of $A=S_t^\dagger M$.

\begin{theorem}
 $A=S_t^\dagger M$ optimizes the objective function $tr\left\{(A^T S_t A)^{\dagger} A^T S_b A\right\}$. 
\end{theorem}
\begin{proof}
We note that $S_b$ can be written as

\begin{equation}
\label{eq:SbMQMT}
\begin{aligned}
S_b &= \sum_{c=1}^C \frac{N_c}{N} \left(\mu_i-\mu\right) \left(\mu_i-\mu\right)^T \\
    &= \frac{1}{N} M \begin{bmatrix} N_1 & 0 & \\ & \ddots & \\ & 0 & N_C \end{bmatrix}	M^T \\
	&= M Q M^T,
\end{aligned}
\end{equation}
where $Q=\frac{1}{N} \begin{bmatrix} N_1 & & \\ & \ddots & \\ & & N_C \end{bmatrix}$.
Now, we show that $A=S_t^\dagger M$ optimizes the objective function. We have
\begin{equation*}
\begin{aligned}
  &tr\left\{(A^T S_t A)^{\dagger} A^T S_b A\right\} \\
= &tr\left\{(M^T S_t^\dagger S_t S_t^\dagger M)^{\dagger} M ^T S_t^\dagger S_b S_t^\dagger M\right\} \\
\stackrel{\dagger}{=} &tr\left\{(M^T S_t^\dagger M)^{\dagger} M^T S_t^\dagger S_b S_t^\dagger M\right\} \\
\stackrel{(\ref{eq:SbMQMT})}{=} &tr\left\{(M^T S_t^\dagger M)^{\dagger} M^T S_t^\dagger M Q M^T S_t^\dagger M\right\} \\
\stackrel{\Circlearrowright}{=} &tr\left\{(M^T S_t^\dagger M)(M^T S_t^\dagger M)^{\dagger} (M^T S_t^\dagger M) Q\right\} \\
\stackrel{\dagger}{=} &tr\left\{(M^T S_t^\dagger M) Q\right\} \\
\stackrel{\Circlearrowright}{=} &tr\left\{S_t^\dagger M Q M^T\right\} \\
\stackrel{(\ref{eq:SbMQMT})}{=} &tr\left\{S_t^\dagger S_b\right\},
\end{aligned}
\end{equation*}
where we have used 1) the key property of the pseudoinverse that for any matrix $A$, $AA^\dagger A = A$, denoted by $\stackrel{\dagger}{=}$, and 2) the cyclic property of trace that $tr\{ABC\}=tr\{CAB\}$, denoted by $\stackrel{\Circlearrowright}{=}$.
The value $tr\left\{S_t^\dagger S_b\right\}$ is the objective value of the original space without any dimensionality reduction which is clearly the largest possible maximum.
Therefore, $A=S_t^\dagger M$ is an optimal solution.
\end{proof}

It may be argued that this solution is not optimal as it extracts $C$ features instead of $C-1$.
In the next theorem we prove that any subset of columns of $A$ containing $C-1$ vectors also optimizes the objective function of LDA.

\begin{theorem}
Any $C-1$ columns of the matrix $A=S_t^\dagger M$ also optimize the objective function $tr\left\{(A^T S_t A)^{\dagger} A^T S_b A\right\}$. 
\label{thm:C-1-solution}
\end{theorem}
\begin{proof}
Without loss of generality, assume that we have removed the $C$-th column. 
Define  
\begin{eqnarray*}
\hat{M}=[\mu_1-\mu,...,\mu_{C-1}-\mu] \\
\hat{Q}=\frac{1}{N} \begin{bmatrix} N_1 & & \\ & \ddots & \\ & & N_{C-1} \end{bmatrix} \\
\hat{N}=\left[\frac{N_1}{N_C},...,\frac{N_{C-1}}{N_C}\right]^T.
\end{eqnarray*}
We have
\begin{align}
&S_b = M Q M^T = \hat{M} \hat{Q} \hat{M}^T + \frac{N_C}{N} (\mu_C-\mu) (\mu_C-\mu)^T \label{eq:sb-based-on-Q-hat}\\
&\mu_C -\mu = -\frac{1}{N_C} \sum_{c=1}^{C-1} N_c (\mu_c-\mu) = -\hat{M} \hat{N},
\label{eq:mu-c-mu}
\end{align}
where the second equation is derived from 
\begin{eqnarray*}
\sum_{c=1}^C N_c \mu_c = N\mu = \left(\sum_{c=1}^C N_c\right)\mu =
\sum_{c=1}^C N_c\mu .
\end{eqnarray*}
Now, we compute the objective function for $\hat{A}=S_t^\dagger \hat{M}$:
\begin{equation*}
\begin{aligned}
  &tr\left\{(\hat{A}^T S_t \hat{A})^{\dagger} \hat{A}^T S_b \hat{A}\right\} \\
= &tr\left\{(\hat{M}^T S_t^\dagger S_t S_t^\dagger \hat{M})^{\dagger} \hat{M}^T S_t^\dagger S_b S_t^\dagger \hat{M}\right\} \\
\stackrel{\dagger}{=} &tr\left\{(\hat{M}^T S_t^\dagger \hat{M})^{\dagger} \hat{M}^T S_t^\dagger S_b S_t^\dagger \hat{M}\right\} \\
\stackrel{(\ref{eq:sb-based-on-Q-hat})}{=} &tr\left\{(\hat{M}^T S_t^\dagger \hat{M})^{\dagger} \hat{M}^T S_t^\dagger \hat{M} \hat{Q} \hat{M}^T S_t^\dagger \hat{M}\right\}\\ 
& + \frac{N_C}{N} 
tr\left\{(\hat{M}^T S_t^\dagger \hat{M})^{\dagger} \hat{M}^T S_t^\dagger \hat{M} \hat{N}\hat{N}^T \hat{M}^T S_t^\dagger \hat{M}\right\}\\
\stackrel{\Circlearrowright,\dagger}{=} &tr\left\{\hat{M}^T S_t^\dagger \hat{M} \hat{Q}\right\} 
+ \frac{N_C}{N} tr\left\{\hat{M}^T S_t^\dagger \hat{M} \hat{N}\hat{N}^T\right\}\\
\stackrel{\Circlearrowright}{=} &tr\left\{S_t^\dagger \hat{M} \hat{Q} \hat{M}^T\right\} 
+ \frac{N_C}{N} tr\left\{S_t^\dagger \hat{M} \hat{N}\hat{N}^T \hat{M}^T\right\}\\
\stackrel{(\ref{eq:mu-c-mu})}{=} &tr\left\{S_t^\dagger \hat{M} \hat{Q} \hat{M}^T\right\} + \frac{N_C}{N} 
tr\left\{S_t^\dagger (\mu_C-\mu) (\mu_C-\mu)^T\right\}\\
\stackrel{(\ref{eq:sb-based-on-Q-hat})}{=} &tr\left\{S_t^\dagger M Q M^T\right\} \\
\stackrel{(\ref{eq:sb-based-on-Q-hat})}{=} &tr\left\{S_t^\dagger S_b\right\},
\end{aligned}
\end{equation*}
where again we have used 1) the key property of the pseudoinverse that for any matrix $A$, $AA^\dagger A = A$, denoted by $\stackrel{\dagger}{=}$, and 2) the cyclic property of trace that $tr\{ABC\}=tr\{CAB\}$, denoted by $\stackrel{\Circlearrowright}{=}$.
This completes the proof since $tr\left\{S_t^\dagger S_b\right\}$ is the maximum attainable objective value and therefore $\hat{A}$ is optimal.
\end{proof}

\subsection{Proposed Numerical Algorithms}
\label{sec:prop-lda-num-alg}
In Section~\ref{sec:lda-num-alg}, we reviewed the classical algorithms for solving LDA in LDLSS and HD/SSS settings.
In this section, we propose two numerical algorithms for computing the solution $S_t^\dagger M$ for the same settings. In LDLSS setting, we can compute the matrix $S_t$ explicitly and then find the least squares solution to the equation $S_t A = M$ as shown in Algorithm~\ref{alg:prop-lda-low-dim}.
For HD/SSS setting, the classical Algorithm~\ref{alg:lda-svd} uses two SVDs to solve LDA. 
Assuming that $H_t=U\Sigma V^T$ is the reduced SVD of $H_t$, we have
\begin{eqnarray*}
\begin{aligned}
S_t^\dagger M = \left(H_t^T H_t\right)^\dagger M 
= \left(V \Sigma^2 V^T\right)^\dagger M 
= V \Sigma^{-2} V^T M.
\end{aligned}
\end{eqnarray*}

Our  method for the HD/SSS setting solves LDA with only one SVD, as shown in Algorithm~\ref{alg:prop-lda-high-dim}. 
However, as our experiments in Section~\ref{sec:experiments-timing} show, the classical Algorithm~\ref{alg:lda-svd} and the proposed Algorithm~\ref{alg:prop-lda-high-dim} have almost identical computational complexity since Algorithm~\ref{alg:prop-lda-high-dim} engages in extra matrix multiplications.

\begin{algorithm}
\caption{Proposed algorithm for LDLSS setting}
\label{alg:prop-lda-low-dim}
\begin{algorithmic}[1]
\STATE Compute $S_t$ explicitly.
\STATE {Construct matrix $M$ whose columns are the means of clusters minus the total mean.}
\STATE Find the least squares solution to $S_t A = M$. 
\RETURN A
\end{algorithmic}
\end{algorithm}

\begin{algorithm}
\caption{Proposed algorithm for HD/SSS setting}
\label{alg:prop-lda-high-dim}
\begin{algorithmic}[1]
\STATE {Compute $H_t$} using (\ref{eq:H-wbt}).
\STATE {Construct matrix $M$ whose columns are the means of clusters minus the total mean.}
\STATE {Compute the reduced SVD to obtain $H_t = U \Sigma V^T$.}
\COMMENT {Eigen-decomposition of $S_t$ is $S_t = V \Sigma^2 V^T$}
\STATE {$A = V \Sigma^{-2} V^T M$}
\RETURN {A}
\end{algorithmic}
\end{algorithm}

\subsection{ Proposed algorithm for Kernel LDA}
\label{sec:prop-lda-kernel}
In this section, we introduce a simpler solution to kernel LDA.
In Section~\ref{sec:lda-kernel}, we stated that the main problem of kernel LDA is (\ref{eq:lda-kernel-problem}), which Baudat and Anouar \cite{baudat2000generalized} further simplified it to 
\begin{eqnarray*}
WK\alpha = \lambda K\alpha.
\end{eqnarray*}
By defining $\beta = K\alpha$, we get
\begin{eqnarray}
\label{eq:W-eigenproblem}
W\beta = \lambda \beta.
\end{eqnarray}

As stated in Section~\ref{sec:prop-speculating-solution}, $W$ is a projection matrix with eigenvalues $0$ and $1$. The unnormalized eigenvectors corresponding to the $C$ eigenvalues with value $1$ are $e^{(1)},...,e^{(c)}$.
For $\beta=e^{(c)}$, where $c\in\left\{1,...,C\right\}$, we obtain the kernel expansion coefficients $\alpha^{(c)}=K^\dagger e^{(c)}$.
Algorithm~\ref{alg:prop-lda-kernel} is our proposed method for solving kernel LDA which is simpler than Algorithm~\ref{alg:lda-kernel}.
Like \citep{baudat2000generalized}, we then normalize the eigenvectors in the feature space by dividing $\alpha$ by $\alpha^T K \alpha$. 
Since $\alpha$ is obtained using pseudoinverse,  generally $\beta\neq K \alpha$ and consequently the eigenvalues of (\ref{eq:W-eigenproblem}) are not equal to those of (\ref{eq:lda-kernel-problem}). However, when the kernel function is strictly positive definite (e.g. RBF), then the matrix K is invertible and the two problems become equivalent. 
In Section~\ref{sec:experiments-iris}, we experimentally show that for strictly positive definite kernels all eigenvalues become $1$.

\begin{algorithm}
\caption{The proposed algorithm for kernel LDA}
\label{alg:prop-lda-kernel}
\begin{algorithmic}[1]
\STATE {Compute the kernel matrix $K$ of training data.}
\STATE {Compute the eigen-decomposition $K=U\Gamma U^T$.}
\FOR {$c=1,2,...,C$}
\STATE $\alpha^{(c)} = U\Gamma^{\dagger} U^T e^{(c)}$
\STATE {Divide $\alpha^{(c)}$ by $\sqrt{{\alpha^{(c)^T} K \alpha^{(c)}}}$ for normalization.}
\ENDFOR
\RETURN {$\alpha^{(1)},...,\alpha^{(C)}$}
\end{algorithmic}
\end{algorithm}

\section{Separating Dimensionality Reduction and Metric Changing}
\label{sec:sep-dim-red-and-metric-changing}
In this section, we show that the EIG-LDA solution is the composition of 1) a dimensionality reduction and 2) a metric changing transformation. We apply a similar metric-changing transformation to LDA++ and obtain EIG-LDA++ which is another optimal solution to LDA.
\subsection{Relation between EIG-LDA and LDA++}
\label{sec:EIG-LDA-LDA++-relation}
As we saw in Section~\ref{sec:lda-num-alg} and Algorithm~\ref{alg:lda-eigen}, an EIG-LDA solution $\tilde{A}$ is obtained by solving the following eigensystem:
\begin{eqnarray}
\begin{aligned}
S_b\tilde{A} = S_t \tilde{A} \Lambda.
\end{aligned}
\label{eq:eig-lda}
\end{eqnarray}
On the other hand, in Theorem~\ref{thm:C-1-solution}, we proved that the LDA++ solution with $C-1$ features $\hat{A}=S_t^\dagger \hat{M}$ is another optimal solution to LDA.
Assume that no optimal solution to LDA has less than $C-1$ features. Then, both $\tilde{A}$ and $\hat{A}$ are full column rank and there exists a non-singular $C-1\times C-1$ square matrix $\hat{Z}$ such that $\tilde{A}=\hat{A}\hat{Z}$. We want to relate EIG-LDA and LDA++ solutions by finding the matrix $\hat{Z}$. Firstly, using (\ref{eq:sb-based-on-Q-hat}) and (\ref{eq:mu-c-mu}), we rewrite $S_b$ as 
\begin{eqnarray}
\begin{aligned}
S_b = \hat{M}\hat{Q}\hat{M}^T + \frac{N_c}{N} \hat{M} \hat{N} \hat{N}^T \hat{M}^T = \hat{M} \tilde{Q} \hat{M}^T,
\end{aligned}
\label{eq:Sb-based-on-Q-tilde}
\end{eqnarray}
where 
\begin{eqnarray}
\begin{aligned}
\tilde{Q} = \hat{Q} + \frac{N_c}{N} \hat{N} \hat{N}^T.
\end{aligned}
\end{eqnarray}
Then, we substitute $S_b$ in (\ref{eq:eig-lda}) and get
\begin{eqnarray}
\begin{aligned}
\hat{M} \tilde{Q} \hat{M}^T \tilde{A} = S_t \tilde{A} \Lambda.
\end{aligned}
\end{eqnarray}
Substituting $\tilde{A}$ with $\hat{A}Z=S_t^\dagger\hat{M}\hat{Z}$ in the above equation we obtain
\begin{eqnarray}
\begin{aligned}
\hat{M} \tilde{Q} \hat{M}^T S_t^\dagger\hat{M}\hat{Z} = S_t S_t^\dagger\hat{M}\hat{Z} \Lambda.
\end{aligned}
\end{eqnarray}
Multiplying the above equation from the left by $S_t^\dagger$ we get
\begin{eqnarray}
\begin{aligned}
S_t^\dagger\hat{M} \tilde{Q} \hat{M}^T S_t^\dagger\hat{M}\hat{Z} = S_t^\dagger\hat{M}\hat{Z} \Lambda.
\end{aligned}
\end{eqnarray}
Since by assumption $\hat{A}=S_t^\dagger\hat{M}$ is full column rank, the above equation simplifies to the eigensystem
\begin{eqnarray}
\begin{aligned}
\tilde{Q} \hat{M}^T S_t^\dagger\hat{M}\hat{Z} = \hat{Z} \Lambda.
\end{aligned}
\label{eq:Z-hat-eigensystem}
\end{eqnarray}
Therefore, the EIG-LDA transformation $Y=\tilde{A}^T X=\hat{Z}^T \hat{A}^T X$ consists of two consecutive transformations. The first transformation is $\hat{A}$ which optimizes LDA, and the second transformation is $\hat{Z}$ which is a non-singular transformation without any effect of the LDA objective function. However, $\hat{Z}$ changes the metric of the feature space and shows its effect when classifying by the nearest neighbor classifier. 
\subsection{EIG-LDA++: LDA++ plus metric learning}
\label{sec:EIG-LDA++}
Learning from the lessons of the previous section, in this section, we add a non-singular metric-changing transformation $Z$ to the $C$-feature LDA++ solution $S_t^\dagger M$ to get $A=S_t^\dagger M Z$.
We propose to find the $C\times C$ matrix $Z$ by solving the following eigenvalue problem
\begin{eqnarray}
\begin{aligned}
Q M^T S_t^\dagger M Z = Z \Lambda,
\end{aligned}
\label{eq:EIG-LDA++}
\end{eqnarray}
which is an extension of (\ref{eq:Z-hat-eigensystem}) to $C\times C$ matrices. We call the solution $S_t^\dagger M Z$ as EIG-LDA++. In our implementation, we first solve for the LDA++ solution $A=S_t^\dagger M$ using either Algorithm~\ref{alg:prop-lda-low-dim} or Algorithm~\ref{alg:prop-lda-high-dim}, whichever appropriate. Then we left-multiply the solution by $QM^T$ to obtain $Q M^T S_t^\dagger M$. Then we solve for $Z$ in (\ref{eq:EIG-LDA++}) and multiply it by the LDA++ solution $A=S_t^\dagger M$. Note that (\ref{eq:EIG-LDA++}) is a $C$-dimensional equation that is solved very fast. In our experiments, we found that, after computing LDA++,  the additional time required to compute EIG-LDA++ is at most several milliseconds.

\section{New findings about classical LDA}
In this section, we mention some new findings about LDA and its relation with multiclass classification. 
In Section~\ref{sec:non-discriminative}, we analyze the optimal solution of LDA++ and discover some previously unknown facts about LDA.
In Section~\ref{sec:relation-with-classification}, we investigate the relation between linear discriminant dimensionality reduction and multiclass classification.
Considering the optimal solution to the classification problem for homoscedastic Gaussian data, given in Section~\ref{sec:bayes-linear-classifier}, we speculate another solution to LDA. We prove that, under some mild conditions, this solution to the problem of linear multiclass classification is also an optimal solution to LDA.
\label{sec:new-findings-lda}
\subsection{Non-discriminative feature-weighting property of LDA}
\label{sec:non-discriminative}
The fact that a solution to LDA only depends on the centers of clusters and the total scatter matrix is astonishing. This shows that, while the objective function of LDA is defined based on the within-class and the between-class scatter matrices, a solution has nothing to do with them and merely depends on the centers of clusters and the total scatter matrix. In other words, LDA simply measures the similarity of the input data with the centers of the clusters using a metric defined by the total scatter matrix, which does not depend on the discriminative task at hand. We call this problem \textit{non-discriminative feature-weighting} and explain this disappointing property of LDA with a real example in Section~\ref{sec:experiments-feret}.
In the extreme case that each training sample becomes a cluster on its own, LDA simplifies to the nearest neighbor algorithm with a metric defined by the total scatter matrix.

In classical pattern recognition textbooks LDA is contrasted against PCA for extracting discriminative, instead of maximum-variance, features.
Quite surprisingly, LDA++, which is an optimal solution to LDA, internally weights features based on PCA, as we now show. Assume that $S_t=\Phi\Lambda \Phi^T$ is the spectral decomposition of $S_t$, where $\Phi$ is the matrix of eigenvectors and $\Lambda$ is the diagonal matrix of eigenvalues. Consider the optimal solution $A=S_t^\dagger M$ for LDA and assume that we want to compute $A^T x$ for some data $x$. We have
\begin{eqnarray}
\begin{aligned}
A^T x &= (S_t^\dagger M)^T x = M^T S_t^\dagger x = M^T (\Phi\Lambda \Phi^T)^\dagger x \\
&= M^T \Phi \Lambda^\dagger \Phi^T x 
= (\Sigma \Phi^T M)^T (\Sigma \Phi^T x) ,
\end{aligned}
\label{eq:lda-as-pca}
\end{eqnarray}
where $\Sigma$ is a diagonal matrix such that $\Sigma^2 = \Lambda^\dagger$. This gives another interpretation for the functionality of LDA++: the features extracted by LDA++ are the dot-product between the weighted PCA features of the centers of clusters and the input data, the weight being the inverse of square root of eigenvalues.
We can also write (\ref{eq:lda-as-pca}) as
$(\Phi \Sigma \Phi^T M)^T (\Phi \Sigma \Phi^T x)$
which brings back data from the PCA space to the input space for visualization.

\subsection{Relation between LDA and multiclass classification}
\label{sec:relation-with-classification}
In Section~\ref{sec:bayes-linear-classifier}, we mentioned that for homoscedastic Gaussian data with nonsingular within-class covariance matrix $\Sigma_w$, the features $\Sigma_w^{-1}M$ are Bayes-optimal for classification. Considering that $S_w$ is the sample within-class covariance matrix, one may wonder if $A_w=S_w^\dagger M$ also optimizes the objective function $J(A)=tr\left\{(A^T S_t A)^{\dagger} A^T S_b A\right\}$? 
In general, $A_w$ is not a solution to LDA as will be shown experimentally by a counterexample in Section~\ref{sec:experiments-uci}. 
However, we will show that when the matrices $S_w$ and $\hat{M}^T S_w^{-1} \hat{M}$ are invertible, then $A_w=S_w^{-1} M$ is also an optimal solution to LDA. Now, we prove a useful lemma.

\begin{lemma}
\label{thm:invertibility}
Assume that the matrices $S_w$ and $R=\hat{M}^T S_w^{-1} \hat{M}$ are invertible and let $\hat{A}_w=\Sigma_w^{-1}\hat{M}$. Then $\hat{A}_w^T S_t \hat{A}_w$ is invertible.
\end{lemma}
\begin{proof}Since $S_w$ is an invertible scatter matrix, it is strictly positive definite. Similarly, $S_b$ is a scatter matrix and therefore is positive semidefinite. It follows that their sum $S_t=S_w+S_b$ is strictly positive definite and therefore invertible.  
Since, by assumption, $\hat{M}^T S_w^{-1} \hat{M}$ is invertible, it follows that $\hat{M}$ is full column rank. Therefore, $\hat{A}_w=\Sigma_w^{-1}\hat{M}$ is the product of a full rank matrix with a full column rank matrix and consequently is full column rank. 
If $\hat{A}_w^T S_t \hat{A}_w$ is not invertible, then there exists a vector $x$ such that $\hat{A}_w^T S_t \hat{A}_w x=0$. Consequently, $x^T \hat{A}_w^T S_t \hat{A}_w x=\|S_t^{1/2} \hat{A}_w x\|=0$. However, this is impossible since both $S_t^{1/2}$ and $\hat{A}_w$ are full column rank. It follows that $\hat{A}_w^T S_t \hat{A}_w$ should be invertible.
\end{proof}

We first prove the optimality of $\hat{A}_w=\Sigma_w^{-1}\hat{M}$ and later conclude that $A_w=S_w^{-1}M$ is also optimal.
\begin{theorem}
\label{thm:SwPinvM-optimality}
If the matrices $S_w$ and $R=\hat{M}^T S_w^{-1} \hat{M}$ are invertible, then $\hat{A}_w=\Sigma_w^{-1}\hat{M}$ is an optimal solution to the objective function $J(A)=tr\{(A^T S_t A)^\dagger A^T S_b A\}$.
\end{theorem}
\begin{proof}
Since, by lemma~\ref{thm:invertibility}, the matrix $\hat{A}_w^T S_t \hat{A}_w$ is invertible, we can get rid of the pseudoinverse and write the objective function as
$J(\hat{A}_w)=tr\left\{(\hat{A}_w^T S_t \hat{A}_w)^{-1} \hat{A}_w^T S_b \hat{A}_w\right\}$. Using \eqref{eq:sb-based-on-Q-hat}, we rewrite $S_b$ as
\begin{equation}
\label{eq:sb-qtilde}
S_b=\hat{M}\hat{Q}\hat{M}^T+\frac{N_c}{N}\hat{M}\hat{N}\hat{N}^T\hat{M}^T=\hat{M}\tilde{Q}\hat{M}^T,
\end{equation}
where
\begin{equation*}
\tilde{Q}=\hat{Q}+\frac{N_C}{N}\hat{N}\hat{N}^T.
\end{equation*}
Substituting $\hat{A}_w$ and $S_b$ in the objective function, we have
\begin{equation*}
\begin{aligned}
  J(\hat{A}_w)=&tr\left\{(\hat{A}_w^T S_t \hat{A}_w)^{-1} \hat{A}_w^T S_b \hat{A}_w\right\} \\
= &tr\left\{(\hat{M}^T S_w^{-1} (S_w+S_b) S_w^{-1} \hat{M})^{{-1}} \hat{M}^T S_w^{-1} S_b S_w^{-1} \hat{M}\right\} \\
= &tr\left\{(\hat{M}^T S_w^{-1} \hat{M} + \hat{M}^T S_w^{-1}\hat{M}\tilde{Q}\hat{M}^T S_w^{-1} \hat{M})^{{-1}} \hat{M}^T S_w^{-1} \hat{M}\tilde{Q}\hat{M}^T S_w^{-1} \hat{M}\right\}.
\end{aligned}
\end{equation*}
Using the assumption that $R=\hat{M}^T S_w^{-1} \hat{M}$ is invertible, we have
\begin{equation*}
\begin{aligned}
J(\hat{A}_w)= &tr\left\{(R + R\tilde{Q}R)^{{-1}} R\tilde{Q}R\right\} \\
= &tr\left\{(R (R^{-1}+ \tilde{Q})R)^{{-1}} R\tilde{Q}R\right\} \\
= &tr\left\{R^{-1}(R^{-1}+ \tilde{Q})^{-1}R^{-1} R\tilde{Q}R\right\} \\
= &tr\left\{RR^{-1}(R^{-1}+ \tilde{Q})^{-1}\tilde{Q}\right\} \\
= &tr\left\{(R^{-1}+ \tilde{Q})^{-1}\tilde{Q}\right\}.
\end{aligned}
\end{equation*}
Now, by applying the Sherman-Morison-Woodbury formula, we obtain
\begin{equation*}
\begin{aligned}
J(\hat{A}_w)=&tr\left\{(R - R (\tilde{Q}^{-1} +R)^{-1}R)\tilde{Q}\right\}.
\end{aligned}
\end{equation*}
By substituting $R$ with $\hat{M}^T S_w^{-1} \hat{M}$, we have
\begin{equation*}
\begin{aligned}
J(\hat{A}_w)=&tr\left\{(R - R (\tilde{Q}^{-1} +\hat{M}^T S_w^{-1}\hat{M})^{-1}R)\tilde{Q}\right\}.
\end{aligned}
\end{equation*}
We apply the Sherman-Morison-Woodbury formula for the second time to obtain
\begin{equation*}
\begin{aligned}
J(\hat{A}_w)=&tr\left\{(R - R (\tilde{Q} - \tilde{Q}\hat{M}^T (S_w + \hat{M}\tilde{Q}\hat{M}^T)^{-1}\hat{M}\tilde{Q})R)\tilde{Q}\right\}\\
=&tr\left\{R\tilde{Q} - R\tilde{Q}R\tilde{Q} + R\tilde{Q}\hat{M}^T (S_w + S_b)^{-1}\hat{M}\tilde{Q}R\tilde{Q}\right\}\\
=&tr\left\{R\tilde{Q} - R\tilde{Q}R\tilde{Q} + R\tilde{Q}R\tilde{Q}\hat{M}^T (S_t)^{-1}\hat{M}\tilde{Q}\right\}\\
=&tr\left\{R\tilde{Q} (I-R\tilde{Q} (I-\hat{M}^T S_t^{-1}\hat{M}\tilde{Q})\right\}.
\end{aligned}
\end{equation*}
Again, by substituting $R$ with $\hat{M}^T S_w^{-1} \hat{M}$, we have
\begin{equation*}
\begin{aligned}
J(\hat{A}_w)=&tr\left\{\hat{M}^T S_w^{-1}\hat{M}\tilde{Q} (I-\hat{M}^T S_w^{-1}\hat{M}\tilde{Q} (I-\hat{M}S_t^{-1}\hat{M}\tilde{Q})\right\}\\
=&tr\left\{S_w^{-1}\hat{M}\tilde{Q} (I-\hat{M}^T S_w^{-1}\hat{M}\tilde{Q} (I-\hat{M}S_t^{-1}\hat{M}\tilde{Q})\hat{M}^T \right\}\\
=&tr\left\{S_w^{-1}\hat{M}\tilde{Q} (\hat{M}^T-\hat{M}^T S_w^{-1}\hat{M}\tilde{Q} (\hat{M}^T-\hat{M}S_t^{-1}\hat{M}\tilde{Q}\hat{M}^T) \right\}\\
=&tr\left\{S_w^{-1}\hat{M}\tilde{Q}\hat{M}^T (I-S_w^{-1}\hat{M}\tilde{Q}\hat{M}^T (I- S_t^{-1}\hat{M}\tilde{Q}\hat{M}^T) \right\}.
\end{aligned}
\end{equation*}
Using \eqref{eq:sb-qtilde} and the invertibility of $S_t$, which follows from the invertibility of $S_w$, we have
\begin{equation*}
\begin{aligned}
J(\hat{A}_w)=&tr\left\{S_w^{-1}S_b (I-S_w^{-1}S_b (I- S_t^{-1}S_b) \right\}\\
=&tr\left\{S_w^{-1}S_b (I-S_w^{-1}S_b S_t^{-1}(S_t- S_b) \right\}\\
=&tr\left\{S_w^{-1}S_b (I-S_w^{-1}S_b S_t^{-1}S_w \right\}\\
=&tr\left\{S_w^{-1}S_b-S_w^{-1}S_bS_w^{-1}S_b S_t^{-1}S_w \right\}\\
=&tr\left\{S_w^{-1}S_b-S_w^{-1}S_b S_t^{-1}S_wS_w^{-1}S_b \right\}\\
=&tr\left\{S_w^{-1}S_b-S_w^{-1}S_b S_t^{-1}S_b \right\}\\
=&tr\left\{S_w^{-1}S_b(I-S_t^{-1}S_b) \right\}\\
=&tr\left\{S_w^{-1}S_bS_t^{-1}(S_t-S_b) \right\}\\
=&tr\left\{S_w^{-1}S_bS_t^{-1}S_w \right\}\\
=&tr\left\{S_t^{-1}S_wS_w^{-1}S_b \right\}
=tr\left\{S_t^{-1}S_b \right\},
\end{aligned}
\end{equation*}
which again is the maximum attainable value, since it equals to the objective value of the original space without any dimensionality reduction. Thus, $\hat{A}_w=S_w^{-1} \hat{M}$ is an optimal solution.
\end{proof}

Furthermore, since $A_w=S_w^{-1} M$ has one more feature than $\hat{A}_w=S_w^{-1} \hat{M}$, it also optimizes the objective function of LDA whenever $\hat{M}^T S_w^{-1} \hat{M}$ is invertible.

\section{Experiments}
\label{sec:experiments}
In this section, we compare EIG-LDA, LDA++, and EIG-LDA++ by performing appropriately chosen experiments.
In Section~\ref{sec:experiments-orl}, we perform experiments on the ORL face recognition dataset \citep{samaria1994parameterisation_orl_dataset} which is an example of HD/SSS setting with $10,304$ dimensions.
We visualize the Fisherfaces of LDA++ and see that each Fisherface corresponds to a separate cluster of faces of an identity.
In Section~\ref{sec:experiments-feret}, we consider the task of gender recognition on the FERET dataset \citep{phillips1998feret, phillips2000feret} and complete our previous discussion at Section~\ref{sec:non-discriminative} about the \textit{non-discriminative feature-weighting} property of LDA with a real example. 
Considering the relationship between LDA and multiclass classification, as discussed in Section~\ref{sec:relation-with-classification}, in Section~\ref{sec:experiments-mnist}, we perform experiments on the MNIST digit recognition dataset \citep{lecun1998gradient_mnist_dataset}, which is an example of LDLSS setting with $60,000$ training samples, and depict the cluster centers and Fisherdigits of EIG-LDA and LDA++ along with the positive centers and weights of a $\pm$ED-WTA \citep{ZareiGhiasi2022} single-layer neural network. We show that, as $\pm$ED-WTA \citep{ZareiGhiasi2022} provides a prototype-based interpretation for multiclass classification, similarly, LDA++ provides a prototype-based interpretation for multiclass linear discriminant analysis.

In Section~\ref{sec:experiments-uci}, we perform experiments on ten UCI datasets chosen by Wan et al. \cite{wan2017separability} along with a modified version of the iris dataset to experimentally support our theoretical results obtained in Section~\ref{sec:proposed-method} and to evaluate EIG-LDA, LDA++, and EIG-LDA++ on a variety of datasets with different numbers of classes, features, and samples. 
Considering kernel LDA, in Section~\ref{sec:experiments-iris}, we repeat the experiment of Baudat and Anouar \cite{baudat2000generalized} on the iris dataset and show that the results reported therein contain numerical error and all eigenvalues are actually one, supporting one of our findings about kernel LDA with strictly positive definite kernels.
Finally, in Section~\ref{sec:experiments-timing}, we compare the training time of EIG-LDA, LDA++, and EIG-LDA++ in LDLSS, HD/SSS, and kernel LDA settings.
The LDA++ package along with the codes of the experiments of this section are publicly available at the authors github page\footnote {\url{https://github.com/k-ghiasi/LDA-plus-plus}}.

\subsection{Experiments on ORL face dataset}
\label{sec:experiments-orl}
In this section, we report the results of our experiments on the ORL face recognition dataset \citep{samaria1994parameterisation_orl_dataset}.
This dataset contains $400$ images of size $92\times 112$ and is an example of the HD/SSS setting with dimension $92\times 112=10,304$.
The images are taken from $40$ subjects in $10$ different conditions: with/without glass, varying the lighting, open/closed eyes, and smiling/not smiling.
First, we illustrate the learned feature extractors of EIG-LDA and LDA++ on this dataset. 
We clustered the images of each person into $4$ subclasses and trained LDA algorithms with a regularization parameter of $1$.
Fig.~\ref{fig:orl-filters} shows the filters learned by EIG-LDA and LDA++ along with the centers of the subclasses, which were obtained using the k-means clustering algorithm.
Clearly, the filters learned by LDA++ have much higher interpretability than those of EIG-LDA.

In another experiment, 
we used 10-fold cross-validation to estimate the accuracy of the nearest neighbor classifier on EIG-LDA, LDA++, and EIG-LDA++ features.
In all 10 experiments, all methods optimized the objective function to the same level, with an average value of $43.94925\pm 0.20$. 
The recognition rates for EIG-LDA, LDA++, and EIG-LDA++ are $98.25\pm 2.25\%$, $97.00\pm 2.92\%$, and 
$98.25\pm 2.25\%$ respectively.

\subsection{Gender-Recognition experiment on FERET}
\label{sec:experiments-feret}
In this section, we design an experiment for discriminating between faces of men and women. We want to know which parts of the face contribute more to masculinity and femininity. We perform this experiment on the color FERET dataset \citep{phillips1998feret, phillips2000feret} which contains $11,338$ facial images of size $512\times 768$. This dataset is distributed in two DVDs, from which we use only the images of the first DVD. We extracted the frontal facial images in which the positions of the left and right eyes had been labeled and the subject did not wear glasses, resulting in $822$ male faces and $654$ female faces. 
We aligned faces using the \textit{imutils} python package in a way that the left and right eyes have fixed positions in the image and cropped and resized the images to obtain pictures of size $182\times 182$. 

We clustered each of the male and female faces into $10$ clusters using the k-means algorithm. Figure~\ref{fig:feret-clusters} shows the centers of clusters of male and female faces.
Since the dimension of data is very high, i.e. 
$182\times 182=33,124$, we use Algorithm~\ref{alg:lda-svd} for EIG-LDA and Algorithm~\ref{alg:prop-lda-high-dim} for LDA++. We used a regularization parameter of $10$. The shape of the learned filters for EIG-LDA and LDA++ are shown in 
Fig.~\ref{fig:feret-lda-filters} and Fig.~\ref{fig:feret-ghiasi-lda-filters}, respectively. 

Also, we performed another experiment in which $90\%$ of samples were used for training, and the remaining $10\%$ samples were used for testing. We experimented $10$ times with different random splitting of the training and testing samples. In all of the $10$ runs, EIG-LDA, LDA++, and EIG-LDA++ optimized the objective function to the same value, with an average of $6.109499\pm 0.07387$.
EIG-LDA, LDA++, and EIG-LDA++ obtained the recognition rates of $83.96\pm 4.44\%$, $83.36\pm 3.39\%$, and $84.09\pm 4.64\%$, respectively.

A basic method for measuring the similarity of an input image $x$ to the centers of clusters is computing the dot-product between them. The problem with this approach is that all pixels of the image have equal contribution to the similarity score, making no difference between the pixels around the eye and the background pixels.
The ideal is to weight the pixels according to their importance in the gender-recognition task. 
However, the optimal solution $S_t^\dagger M$ for LDA is disappointing as it weights the pixels based on the total scatter matrix, without any attention to the task at hand (i.e. gender recognition). 
At the end of Section~\ref{sec:non-discriminative}, we introduced another interpretation of LDA++ as the dot-product between the weighted PCA features of the centers of clusters and the input data. 
Figure~\ref{fig:pca-inside-lda} visualizes this interpretation of LDA++ features on the FERET dataset.
This explanation completes the discussion at Section~\ref{sec:non-discriminative} about \textit{non-discriminative feature-weighting} property of LDA.

\subsection{Experiments on MNIST}
\label{sec:experiments-mnist}
In this section, we report the results of our experiments on the MNIST dataset \citep{lecun1998gradient_mnist_dataset}.
This dataset consists of 60,000 training images and 10,000 testing images of size $28\times 28$ from handwritten English digits $0$ to $9$. 
In this experiment, we trained LDA with a regularization parameter of $1$.
We clustered the images of each digit into $6$ subclasses using the k-means clustering algorithm.
Fig.~\ref{fig:mnist-filters}-(a) illustrates the centers of these clusters. 
We then performed linear discriminant analysis on these $60$ clusters.
Fig.~\ref{fig:mnist-filters}-(b) 
and
Fig.~\ref{fig:mnist-filters}-(c)
visualize the feature extractors obtained by EIG-LDA and LDA++, respectively.
It can be seen that although both the classical and the proposed solvers of LDA optimized the objective function of LDA to the same value of $7.347679$, the $C$ filters obtained by LDA++ are much more interpretable than the $C-1$ filters obtained by EIG-LDA.
Using the nearest neighbor classifier, the recognition rates of EIG-LDA, LDA++, and EIG-LDA++ were $97.24\%$, $97.09\%$, and $97.26\%$ ,respectively. 

We also trained a $\pm$ED-WTA network \citep{ZareiGhiasi2022} with $6$ neurons for each class.
$\pm$ED-WTA is a novel single-layer neural network which yields interpretable prototypes for each neuron in addition to the weights and biases.
We initialized the positive centers of the $\pm$ED-WTA network by the centers obtained by the k-means clustering algorithm and the negative centers with the mean of all training data. We trained the network for $100$ epochs with an initial learning rate of $0.01$ and a weight decay of $0.0001$ times the learning rate. The learning rate was decayed at each epoch by a factor of $0.95$. 
The final values of the positive centers and the weights are visualized in Fig.~\ref{fig:mnist-filters-pmEDWTA}. 
The neural network obtained an accuracy of $96.68\%$ using its built-in winner-take-all classifier and an accuracy of $97.13\%$ using the nearest neighbor classifier. 
Comparison of Fig.~\ref{fig:mnist-filters} and Fig.~\ref{fig:mnist-filters-pmEDWTA} reveals the similarity between the weights of LDA and a winner-takes-all classifier and the prototypes associated with them.

\subsection{Experiments on UCI datasets}
\label{sec:experiments-uci}
In this section, we evaluate the proposed method on 10 UCI datasets chosen by Wan et al. \cite{wan2017separability}.
The information of these datasets is tabulated in Table~\ref{tbl:uci-datasets}. We draw the attention of the reader to the leaf dataset in which the number of features (i.e. 14) is less than the number of classes (i.e. 30). 
This shows that LDA is detoured from its original goal of dimensionality reduction to a combination of representation and metric learning.
In fact, by considering subclass structure, as was done in \citep{wan2017separability}, the number of extracted features from many other datasets (e.g. iris, red wine quality, breast tissue, seeds, and banknote) would be more than the number of input attributes.
We also constructed a new dataset which we call 'singular isis'.
The 'singular isis' dataset is the same as the isis dataset with the single difference that the input feature vector is augmented with the class label. Since the samples of each class have the same value for this new feature, the within-class scatter matrix becomes singular.

In all datasets, the number of samples is more than the number of attributes (i.e. LDLSS setting). 
We had an option to choose either Algorithm~\ref{alg:prop-lda-low-dim} or Algorithm~\ref{alg:prop-lda-high-dim}. We experimented with both algorithms and the results were identical. In this section, LDA++ refers to either Algorithm~\ref{alg:prop-lda-low-dim} or Algorithm~\ref{alg:prop-lda-high-dim}. 
We also performed experiments with EIG-LDA++ and $A_w=S_w^\dagger M$.
We split the samples of each dataset into 10 folds.
To estimate the recognition rate, we used 9 folds for training and the remaining fold for testing.
We then reported the average recognition rates on the 10 possible ways of choosing the training and testing folds.
Table~\ref{tbl:uci-objective} reports the objective values attained by EIG-LDA, LDA++, EIG-LDA++, and $A_w=S_w^\dagger M$. As can be seen, in all datasets, EIG-LDA, LDA++ , and EIG-LDA++ have attained the same objective value in all 10 runs (except a negligible difference in 'urban land').
Besides, in all experiments, $A_w=S_w^\dagger M$ has been an optimal solution, except in the 'singular iris' in which the within-class scatter matrix is singular and the premises of Theorem~\ref{thm:SwPinvM-optimality} do not hold.
Table~\ref{tbl:uci-accuracy} reports the average recognition rates of EIG-LDA, LDA++, EIG-LDA++, and $A_w=S_w^\dagger M$ when classified using the nearest neighbor classifier. Although all methods optimize the objective function of LDA to the same level, their performances on the nearest neighbor classifier differ.

To check that all methods are indeed finding the same subspace, we passed the filters learned by them through a QR decomposition and used the columns of $Q$ as feature extractors to find an orthonormal basis for the subspace acquired by each method. 
If the methods have learned the same subspace, then the learned matrices $Q$ would be two orthonormal bases for the same subspace and the accuracy of the nearest neighbor classifier on the features extracted by the $Q$ matrices would be the same. 
Table~\ref{tbl:uci-accuracy-qr} tabulates the accuracy of the nearest neighbor classifier on an orthonormal basis for the subspace learned by EIG-LDA and LDA++. As expected, all methods yield the same accuracies on all datasets except for $A_w=S_w^\dagger M$ on 'singular iris' which was expected. This verifies that EIG-LDA and LDA++ find two different non-orthogonal bases for the same subspace.

\subsection{Comparison with KFDA on Iris}
\label{sec:experiments-iris}
In this section, we repeat the experiment of Baudat and Anouar\cite{baudat2000generalized} on the iris dataset and mention some interesting facts. The iris dataset was originally introduced by Fisher \cite{fisher1936use}.
It consists of 150 samples with 4 features from 3 classes, each class containing 50 samples.
Baudat and Anouar \cite{baudat2000generalized} trained kernel LDA with a Gaussian kernel with $\sigma=0.7$ and reported that the first two eigenvalues were $0.999$ and $0.985$. 
However, as mentioned in Section~\ref{sec:prop-lda-kernel}, for strictly positive definite kernels, all non-zero eigenvalues ought to be $1$. 
In fact, the inaccuracy in finding eigenvalues were introduced in their implementation because of performing an eigen-decomposition of the kernel matrix and throwing away small eigenvalues and their corresponding eigenvectors\footnote{We are thankful to  Baudat and Anouar \cite{baudat2000generalized} for making their implementation publicly available. This detailed analysis was impossible without having access to their code.}.
We repeated this experiment and solved it both with Algorithm~\ref{alg:lda-kernel} of Baudat and Anouar \cite{baudat2000generalized} and Algorithm~\ref{alg:prop-lda-kernel} proposed in this paper.
Both methods found that the first two eigenvalues are $1$ and the objective function is $2$ (up to 12 digits after the decimal point).

\subsection{Timing experiments}
\label{sec:experiments-timing}
In this section, we compare the training time of EIG-LDA and LDA++.
To have the freedom to manipulate the number of training samples and the input dimension, we constructed an artificial dataset with 3 classes, where data for each class comes from a Gaussian distribution.
The first two features are generated from a normal distribution with covariance matrix 
$\begin{bmatrix} 4.625 & 4.375 \\ 4.375 & 4.625\end{bmatrix}$.
For the three classes, the mean of the first two features are $[-5,-5]$, $[0,0]$, and $[5,5]$.
The rest of the features are generated from a normal distribution with mean $0$ and standard deviation $0.5$.
The timing experiments of this section have been performed on a UX310UQ laptop computer with an Intel(R) Core(TM) i7-6500U CPU and 12GB memory.

We designed two experiments for the LSLSS and HD/SSS scenarios. 
For LDLSS scenario, we used $N=36,000$ samples, consisting of $12,000$ samples for each class. 
Table~\ref{tbl:timing-scenario-1} compares the running time of the classical Algorithm~\ref{alg:lda-eigen} with the proposed Algorithm~\ref{alg:prop-lda-low-dim}. 
As can be seen, initially Algorithm~\ref{alg:prop-lda-low-dim} is faster but as the number of features increases, the classical Algorithm~\ref{alg:lda-eigen} becomes faster.
Note that the dimensionality of the data cannot be increased arbitrarily since then the problem switches from LDLSS to HD/SSS which should be solved by eiher Algorithms~\ref{alg:lda-svd} or Algorithms~\ref{alg:prop-lda-high-dim}.
For the HD/SSS scenario, we used $N=900$ samples, consisting of $300$ samples for each class. 
Table~\ref{tbl:timing-scenario-2} compares the running time of the classical Algorithm~\ref{alg:lda-svd} with the proposed Algorithm~\ref{alg:prop-lda-high-dim}. 
In this experiment, we observe that Algorithm~\ref{alg:prop-lda-high-dim} is slightly faster. 
We also have reported the value of the objective function at the solutions which shows that both classical and the proposed algorithms have attained identical objective values.

In our last experiment, we compare the training time of classical and the proposed kernel LDA algorithms. 
We used $N=900$ samples, where the number of samples of each class is $300$. 
We used the Gaussian kernel function $k(x,z)=exp{\left(-\|x-z\|^2/\sigma^2\right)}$ with $\sigma^2=10$.
Table~\ref{tbl:timing-scenario-3} compares the running time of the classical Algorithm~\ref{alg:lda-kernel} with the proposed Algorithm~\ref{alg:prop-lda-kernel}.
We observe that Algorithm~\ref{alg:prop-lda-kernel} is slightly faster. 
Considering that for Gaussian kernel the kernel matrix is invertible, all eigenvalues are one and the objective function in all entries of Table~\ref{tbl:timing-scenario-3} is $2.0$. 
Overall, it can be said that for practical purposes, the running time of the classical and the proposed solvers are similar. 

\section{Conclusion}
\label{sec:conclusion}
In this paper, we revisited classical  multiclass LDA and introduced a novel interpretable solution to LDA, called LDA++, which does not pass through solving an eigensystem. We considered a modernized variant of classical LDA in which
\begin{itemize}
\item the objective function is generalized from $(A^T S_t A)^{-1}(A^T S_b A)$ to $(A^T S_t A)^\dagger(A^T S_b A)$  \cite{ye2005characterization},
\item each class is partitioned into several clusters and the analysis is performed on clusters instead of classes,
\item the analysis is applicable to the HD/SSS setting in which the scatter matrices cannot be computed or well estimated, and
\item the analysis may be performed in a feature space associated with a positive definite kernel function.
\end{itemize}

In contrast to LDA, which extracts at most $C-1$ features, where $C$ stands for the number of clusters, LDA++ extracts $C$ interpretable features, where each feature shows the similarity to a cluster.
We proved that any $(C-1)$-subset of these features is also an optimal solution to LDA.
Considering that a multi-prototype multiclass classifier also computes a score function for each of the $C$ clusters, we investigated a relationship between LDA and multiclass classification. We proved that, under some mild conditions, the optimal weights of a multiclass linear classifier for homoscedastic Gaussian data are also optimal discriminative features for LDA. Again, we proved that any $(C-1)$-subset of these features is also an optimal solution to LDA.
We introduced two numerical algorithms for LDLSS and HD/SSS settings.

Considering that both EIG-LDA and LDA++ with $C-1$ features are solutions to classical LDA, they should be related by a non-singular metric-changing matrix. We explicitly find this metric-changing transformation for EIG-LDA and propose a similar metric-changing transformation for LDA++ with $C$ features and arrive at another solution called EIG-LDA++.

LDA++ revealed that, in contrast to the common belief, an optimal solution to LDA has nothing to do with the within-class and between-class scatter matrices, and can be described merely based on the centers of clusters and the total scatter matrix.
Specifically, each feature of LDA++ is the dot-product between the weighted PCA features of a cluster center and the input.

Besides, we observed that different optimal solutions to LDA have different recognition rates when used with the nearest neighbor classifier
This observation is consistent with our theoretical understanding, since LDA seeks feature extraction for the Bayes classifier, not the nearest neighbor \citep{fukunaga1990statistical}.
This shows that comparing new LDA methods based on their performance on the nearest neighbor classifier has detoured research on linear discriminate analysis from dimensionality reduction to metric learning. 

The main limitation of the proposed interpretable solution is that, at the moment, it is only applicable to the objective function of the classical LDA and the generalization proposed by Ye \cite{ye2005characterization}. 
As a result, LDA++ inherits all limitations of classical LDA, including limited applicability for extremely non-Gaussian distributed data and inferior performance compared to more sophisticated objective functions. 
Obtaining an interpretable solution to other variants of LDA is a potential direction for future research.

\section*{Acknowledgement}
Portions of the research in this paper use the FERET database of facial images collected under the FERET program, sponsored by the DOD Counterdrug Technology Development Program Office.

\begin{figure}[h]
\centering
\subfloat[Visualization of the first 40 filters from 159 filters learned by EIG-LDA.]{
\includegraphics[width=0.75\columnwidth]{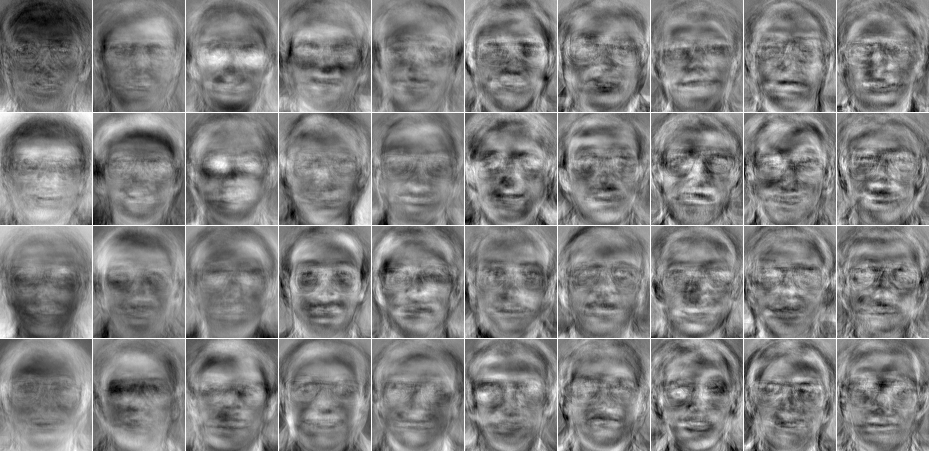}
\label{}
}
\vfill
\centering
\subfloat[Visualization of the first 40 filters from 160 filters learned by LDA++.]{
\includegraphics[width=0.75\columnwidth]{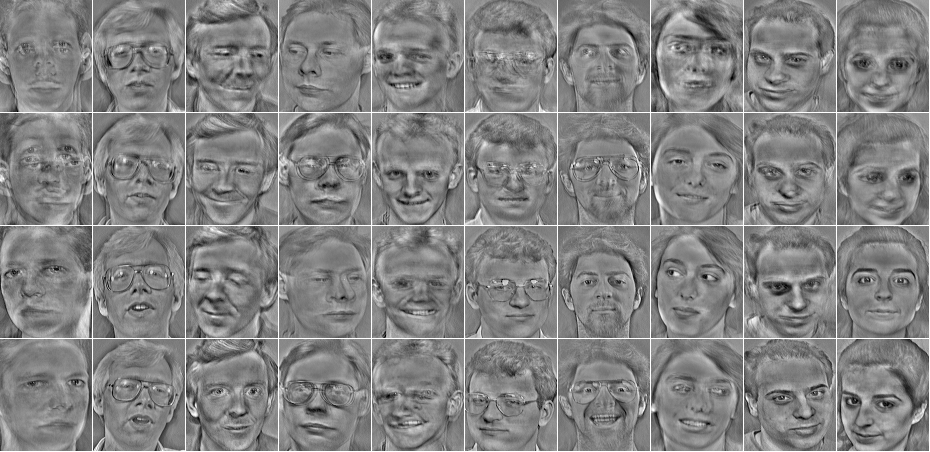}
\label{}
}
\vfill
\centering
\subfloat[Centers of $4$ subclasses of face images of the first $10$ subjects of the ORL dataset.]{
\includegraphics[width=0.75\columnwidth]{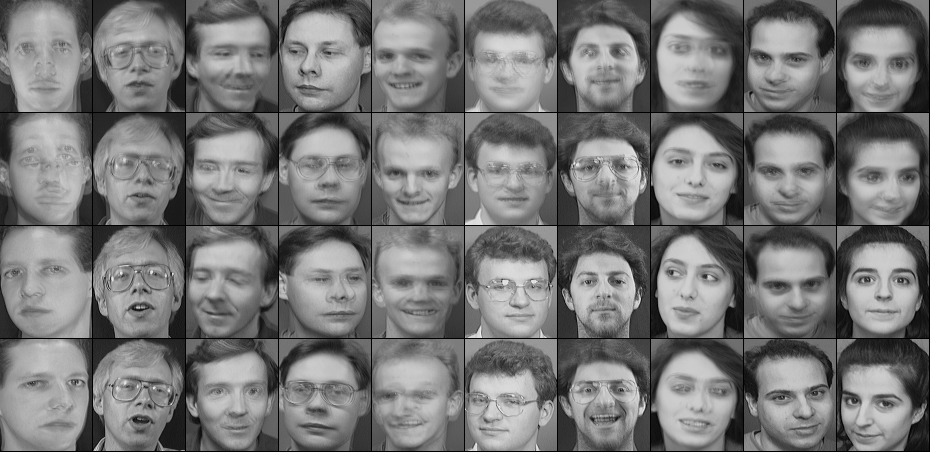}
\label{fig:orl-cluster}
}
\vfill
\centering
 \caption{Visualization of some of the linear feature extractors learned by (a) EIG-LDA and (b) LDA++ on the ORL dataset. While EIG-LDA extracted $C-1=159$ features, LDA++ extracts $C=160$ features, where each one is associated with a cluster. The cluster centers are depicted in (c). Both solvers optimized the objective function of LDA to the same value of $42.44.480186$.  
For better visualization, only the first $40$ feature extractors are shown.}
  \label{fig:orl-filters}
\end{figure}

\begin{figure}[h]
\centering
\subfloat[Visualization of the centers of the $10$ male and $10$ female clusters learned by the k-means algorithm.]{
\label{fig:feret-clusters}
\includegraphics[width=1\columnwidth]
{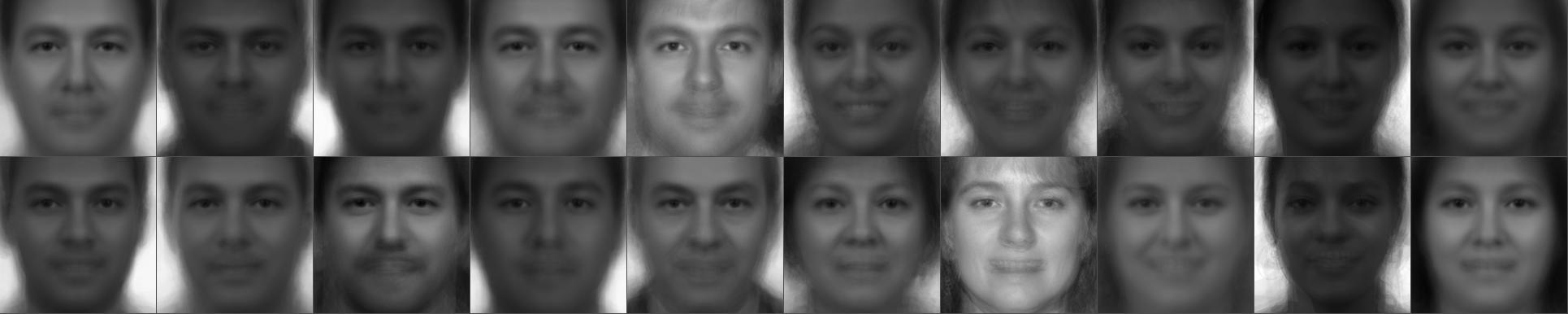}
}
\vfill
\centering
\subfloat[The transformations learned by EIG-LDA.]{
\includegraphics[width=1\columnwidth]
{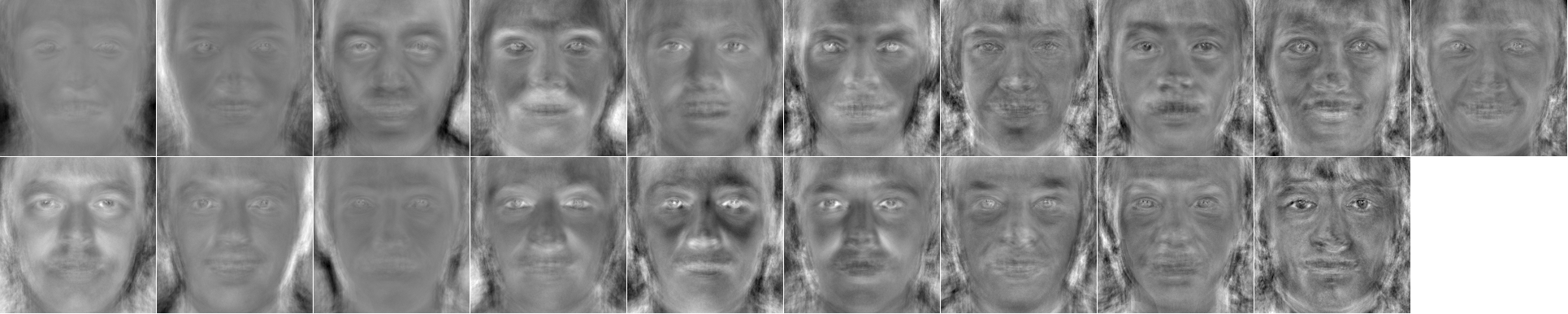}
\label{fig:feret-lda-filters}
}
\vfill
\centering
\subfloat[The transformations learned by LDA++.]{
\includegraphics[width=1\columnwidth]
{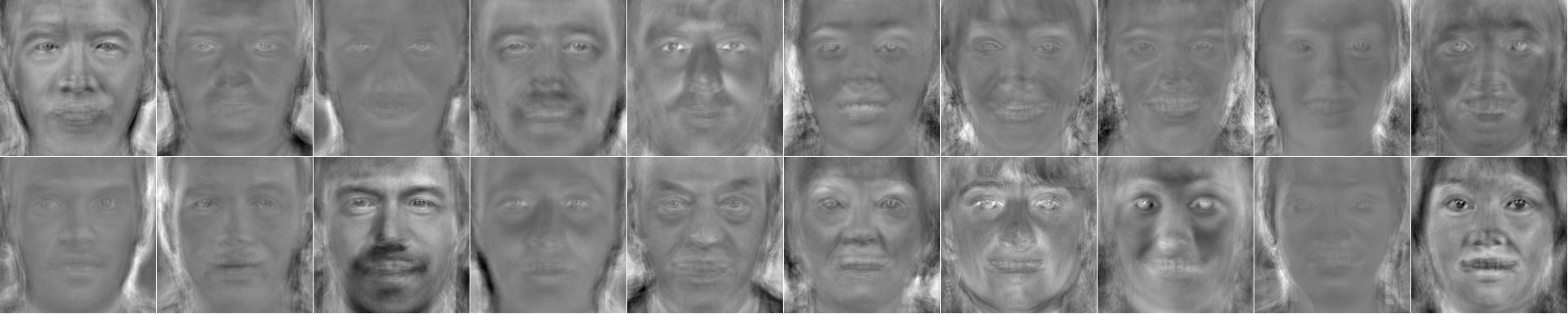}
\label{fig:feret-ghiasi-lda-filters}
}
\vfill
 \caption{Visualization of EIG-LDA and LDA++ for the task of gender recognition on the FERET dataset. Both solutions obtained the optimal objective value $4.09324$.}
  \label{fig:feret-filters}
\vspace{0.5cm}
\centering
\includegraphics[width=1\columnwidth]{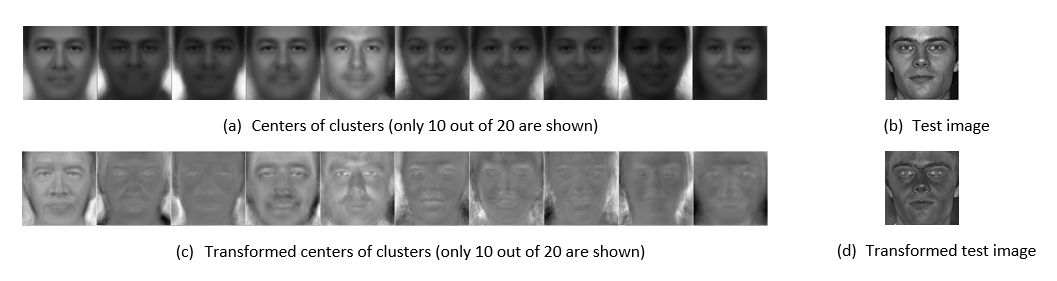}
 \caption{Another interpretation of the features extracted by LDA++. Both centers of clusters and test images are transformed by the PCA and weighted by the inverse of the square root of eigenvalues of $S_t$. 
 For visualization purposes, PCA features are transformed back into the input space. The features extracted by LDA++ are the dot-product between the transformed test image and transformed centers. This finding shows that an optimal solution to LDA internally uses PCA for weighting features, and the supervised labels are only used in finding cluster centers.}
  \label{fig:pca-inside-lda}
\end{figure}

\begin{figure}[h]
\centering
\subfloat[The transformations learned by EIG-LDA.]{
\includegraphics[width=0.45\columnwidth]{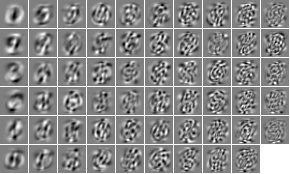}
\label{}
}
\vfill
\centering
\subfloat[Centers of the $60$ clusters learned by k-means. These are prototypes associated with features of LDA++.]{
\includegraphics[width=0.45\columnwidth]
{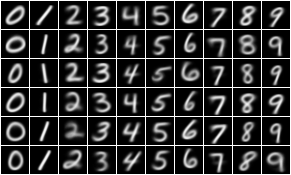}
\label{}
}
\hfill
\centering
\subfloat[The transformations learned by LDA++.]{
\includegraphics[width=0.45\columnwidth]{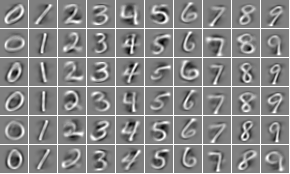}
\label{}
}
\hfill
 \caption{Visualization of EIG-LDA and LDA++ on the MNIST dataset when each of the $10$ digits is clustered into $6$ subclasses, yielding a total of $C=60$ clusters. Both solutions obtained the optimal objective value $7.347679$.
Using the nearest neighbor classifier, EIG-LDA and LDA++ obtained recognition rates of $97.24\%$ and $97.09\%$, respectively. 
 }
  \label{fig:mnist-filters}
\vspace{0.5cm}  
\centering
\subfloat[Positive centers associated with each neuron of a $\pm$ED-WTA network trained on MNIST.]{
\includegraphics[width=0.45\columnwidth]{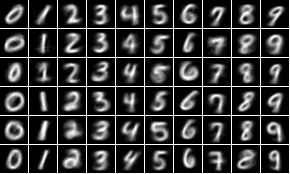}
\label{}
}
\hfill
\centering
\subfloat[Weights of neurons of a $\pm$ED-WTA network trained on MNIST.]{
\includegraphics[width=0.45\columnwidth]
{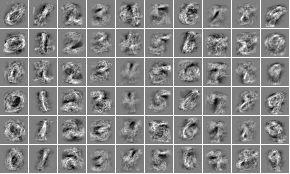}
\label{}
}
\hfill
\caption{Visualization of positive centers and the weights of a single-layer neural network, modeled as $\pm$ED-WTA \citep{ZareiGhiasi2022},  with $6$ output neurons for each class, trained for $100$ epochs on the MNIST dataset. 
The positive centers have been initialized to the cluster centers depicted in Fig.~\ref{fig:mnist-filters}(a) and the negative centers have been initialized to the mean of all training samples. 
The recognition accuracy of the winner-take-all and the nearest neighbor classifiers are $96.68\%$ and $97.13\%$, respectively.
}
  \label{fig:mnist-filters-pmEDWTA}
\end{figure}

\begin{table}[t]
\renewcommand{\arraystretch}{1}
\caption{Information about the 10 UCI datasets chosen by Wan et al. \cite{wan2017separability}}
\label{tbl:uci-datasets}
  \centering
  \begin{tabular}{l|r|r|r}
    \hline
    \multicolumn{1}{c|}{\small\textbf{Dataset}} &
    \multicolumn{1}{c|}{\small\textbf{\#Samples}} &
    \multicolumn{1}{c|}{\small\textbf{\#Features}} & 
    \multicolumn{1}{c}{\small\textbf{\#Classes}} \\
    \hline
banknote  &  1372 & 4 & 2\\
breast tissue & 106 & 9 & 6 \\
forest types &  523 & 27 & 4\\
iris  &  150 & 4 & 3\\
leaf  &  340 & 14 & 30\\
red wine quality  & 1599 & 11 & 6 \\
seeds  & 210 & 7 & 3 \\
urban land  &  675 & 147 & 9\\
vehicle  &  846 & 18 & 4\\
wdbc  &  569 & 30 & 2\\
    \hline
  \end{tabular}  
\end{table}
\begin{table}[!t]
\renewcommand{\arraystretch}{1}
\caption{Average objective values attained by EIG-LDA ,LDA++, EIG-LDA++, and  $S_w^\dagger M$ on the selected UCI datasets. The slightly lower reported value for LDA++ on urban\_land is a numerical  inaccuracy as is evident from the value of EIG-LDA++ on the same dataset.}
\label{tbl:uci-objective}
  \centering
\begin{tabular}{l|c|c|c|c}
    \hline
    \multicolumn{1}{c|}{\small\textbf{Dataset}} &
    \multicolumn{1}{c|}{\small\textbf{EIG-LDA}} &
    \multicolumn{1}{c|}{\small\textbf{LDA++}} &
    \multicolumn{1}{c|}{\small\textbf{EIG-LDA++}}&
    \multicolumn{1}{c}{\small{$\mathbf{S_w^\dagger M}$}} \\
    \hline
banknote  & $ 0.86 $  & $ 0.86 $  & $ 0.86 $  & $ 0.86 $ \\
breast\_tissue  & $ 2.13 $  & $ 2.13 $  & $ 2.13 $  & $ 2.13 $ \\
forest\_types  & $ 1.86 $  & $ 1.86 $  & $ 1.86 $  & $ 1.86 $ \\
iris  & $ 1.19 $  & $ 1.19 $  & $ 1.19 $  & $ 1.19 $ \\
leaf  & $ 7.90 $  & $ 7.90 $  & $ 7.90 $  & $ 7.90 $ \\
rwq  & $ 0.50 $  & $ 0.50 $  & $ 0.50 $  & $ 0.50 $ \\
seeds  & $ 1.61 $  & $ 1.61 $  & $ 1.61 $  & $ 1.61 $ \\
urban\_land  & $ 5.28 $  & $ 5.27 $  & $ 5.28 $  & $ 5.28 $ \\
vehicle  & $ 1.51 $  & $ 1.51 $  & $ 1.51 $  & $ 1.51 $ \\
wdbc  & $ 0.78 $  & $ 0.78 $  & $ 0.78 $  & $ 0.78 $ \\
singular\_iris  & $ 1.66 $  & $ 1.66 $  & $ 1.66 $  & $\mathbf{1.19}$ \\

    \hline
  \end{tabular}  
\end{table}
\begin{table}[!t]
\renewcommand{\arraystretch}{1}
\caption{Average recognition rates (in percent) of the nearest neighbor classifier using EIG-LDA, LDA++, EIG-LDA++, and $S_w^\dagger M$  on the selected UCI datasets.}
\label{tbl:uci-accuracy}
  \centering
\begin{tabular}{l|r|r|r|r}
    \hline
    \multicolumn{1}{c|}{\small\textbf{Dataset}} &
    \multicolumn{1}{c|}{\small\textbf{EIG-LDA}} &
    \multicolumn{1}{c|}{\small\textbf{LDA++}} &
    \multicolumn{1}{c}{\small\textbf{EIG-LDA++}} &
    \multicolumn{1}{c}{\small{$\mathbf{S_w^\dagger M}$}} \\    
    \hline

banknote  & $ 99.57 \pm  0.48 $  & $ 99.57 \pm  0.48 $  & $ 99.93 \pm  0.22 $  & $ 99.57 \pm  0.48 $ \\
breast\_tissue  & $ 67.14 \pm  12.45 $  & $ 66.43 \pm  13.19 $  & $ 67.86 \pm  11.18 $  & $ 61.43 \pm  12.04 $ \\
forest\_types  & $ 85.00 \pm  4.26 $  & $ 86.11 \pm  4.32 $  & $ 85.19 \pm  4.83 $  & $ 85.00 \pm  2.10 $ \\
iris  & $ 94.67 \pm  4.99 $  & $ 94.00 \pm  4.67 $  & $ 95.33 \pm  5.21 $  & $ 96.00 \pm  5.33 $ \\
leaf  & $ 79.23 \pm  3.63 $  & $ 76.92 \pm  5.70 $  & $ 76.92 \pm  5.70 $  & $ 74.81 \pm  5.33 $ \\
rwq  & $ 63.09 \pm  3.28 $  & $ 63.83 \pm  2.38 $  & $ 65.06 \pm  4.76 $  & $ 64.01 \pm  2.24 $ \\
seeds  & $ 96.19 \pm  3.56 $  & $ 97.14 \pm  3.16 $  & $ 96.19 \pm  4.67 $  & $ 96.67 \pm  3.05 $ \\
urban\_land  & $ 77.92 \pm  5.14 $  & $ 75.00 \pm  5.86 $  & $ 77.92 \pm  5.25 $  & $ 79.17 \pm  6.15 $ \\
vehicle  & $ 75.93 \pm  4.77 $  & $ 75.47 \pm  4.30 $  & $ 75.81 \pm  2.79 $  & $ 74.88 \pm  3.57 $ \\
wdbc  & $ 94.66 \pm  3.97 $  & $ 94.66 \pm  3.97 $  & $ 95.52 \pm  2.46 $  & $ 94.66 \pm  3.97 $ \\
singular\_iris  & $ 100.00 \pm  0.00 $  & $ 100.00 \pm  0.00 $  & $ 100.00 \pm  0.00 $  & $ 96.00 \pm  5.33 $ \\

    \hline
  \end{tabular}  
\end{table}
\begin{table}[!t]
\renewcommand{\arraystretch}{1}
\caption{Average recognition rates (in percent) of the nearest neighbor classifier using EIG-LDA, LDA++, EIG-LDA++, and $S_w^\dagger M$ on the selected UCI datasets when the columns of the dimensionality reduction matrix are orthogonalized using the QR decomposition.}
\label{tbl:uci-accuracy-qr}
  \centering
  \begin{tabular}{l|r|r}
    \hline
    \multicolumn{1}{c|}{\small\textbf{Dataset}} &
    \multicolumn{1}{c|}{
    $\begin{array}{c} \small\textbf{EIG-LDA} \\ \small\textbf{LDA++}\\
\small\textbf{EIG-LDA++}\end{array}$} &
    \multicolumn{1}{c}{\small\textbf{$\mathbf{S_w^\dagger M}$}} \\
    \hline
banknote  & $ 99.57 \pm  0.48 $  & $ 99.57 \pm  0.48 $ \\
breast\_tissue  & $ 67.86 \pm  13.27 $  & $ 67.86 \pm  13.27 $ \\
forest\_types  & $ 84.44 \pm  3.33 $  & $ 84.44 \pm  3.33 $ \\
iris  & $ 94.67 \pm  4.99 $  & $ 94.67 \pm  4.99 $ \\
leaf  & $ 75.0 \pm  7.25 $  & $ 75.0 \pm  7.25 $ \\
rwq  & $ 62.22 \pm  4.05 $  & $ 62.22 \pm  4.05 $ \\
seeds  & $ 96.19 \pm  3.56 $  & $ 96.19 \pm  3.56 $ \\
urban\_land  & $ 56.39 \pm  5.39 $  & $ 56.39 \pm  5.39 $ \\
vehicle  & $ 75.7 \pm  4.27 $  & $ 75.7 \pm  4.27 $ \\
wdbc  & $ 94.66 \pm  3.97 $  & $ 94.66 \pm  3.97 $ \\
singular\_iris  & $ 100.00 \pm  0.00 $  & ${ \mathbf{94.67 \pm  4.99 }}$ \\
    \hline
  \end{tabular}  
\end{table}

\begin{table}[!t]
\renewcommand{\arraystretch}{1}
\caption{Average (over five runs) training time (in seconds) of EIG-LDA (using Algorithm~\ref{alg:lda-eigen}) and LDA++ and EIG-LDA++ (using Algorithm~\ref{alg:prop-lda-low-dim}) in LDLSS setting with $N=36,000$ samples.}
\label{tbl:timing-scenario-1}
  \centering
  \begin{tabular}{r|r|r|r}
    \hline
    {\small\textbf{Dim}} &
    \multicolumn{1}{c|}{\small\textbf{EIG-LDA}} &
    \multicolumn{1}{c|}{\small\textbf{LDA++}} &    
	\multicolumn{1}{c}{\small\textbf{EIG-LDA++}} \\ 
    \hline
256 & 0.1644 & 0.1622 & 0.1623 \\
512 & 0.4502 & 0.5075 & 0.5076 \\
1024 & 1.6411 & 1.9415 & 1.9416 \\
2048 & 6.6982 & 8.4851 & 8.4853 \\
4096 & 32.6432 & 45.1203 & 45.1207 \\
8192 & 189.1050 & 276.0101 & 276.0114 \\
    \hline
  \end{tabular}  
\end{table}

\begin{table}[!t]
\renewcommand{\arraystretch}{1}
\caption{
Average (over five runs) training time (in seconds) of EIG-LDA (using Algorithm~\ref{alg:lda-svd}) and LDA++ and EIG-LDA++ (using Algorithm~\ref{alg:prop-lda-high-dim}) in HD/SSS setting with $N=900$ samples.}
\label{tbl:timing-scenario-2}
  \centering
  \begin{tabular}{r|r|r|r}
    \hline
    {\small\textbf{Dim}} &
    \multicolumn{1}{c|}{\small\textbf{EIG-LDA}} &
    \multicolumn{1}{c|}{\small\textbf{LDA++}} &    
	\multicolumn{1}{c}{\small\textbf{EIG-LDA++}} \\ 
    \hline
1024 & 0.8183 & 0.8025 & 0.8026 \\
2048 & 1.2894 & 1.2649 & 1.2651 \\
4096 & 1.8723 & 1.8565 & 1.8567 \\
8192 & 3.2417 & 3.2404 & 3.2407 \\
16384 & 6.8361 & 6.7349 & 6.7354 \\
32768 & 14.3921 & 14.9536 & 14.9545 \\
    \hline
  \end{tabular}  
\end{table}

\begin{table}[!t]
\renewcommand{\arraystretch}{1}
\caption{Average (over five runs) training time (in seconds) of the classical Algorithm~\ref{alg:lda-kernel} and the novel Algorithm~\ref{alg:prop-lda-kernel} for kernel LDA with $N=900$ samples. In all experiments, the optimal objective value $2.0000$ has been obtained.}
\label{tbl:timing-scenario-3}
  \centering
  \begin{tabular}{r|r|r}
    \hline
    {\small\textbf{Dim}} &
    \multicolumn{1}{c|}{\small\textbf{Algorithm~\ref{alg:lda-kernel}}} &
	\multicolumn{1}{c}{\small\textbf{Algorithm~\ref{alg:prop-lda-kernel}}} \\ 
    \hline	
1024 & 0.9124 & 0.7883 \\
2048 & 1.0109 & 0.9128 \\
4096 & 1.3693 & 1.2194 \\
8192 & 2.0191 & 1.9004 \\
16384 & 8.7322 & 8.5746 \\
32768 & 11.2022 & 11.1802 \\
    \hline
  \end{tabular}  
\end{table}


\begin{thebibliography}{10}
\expandafter\ifx\csname url\endcsname\relax
  \def\url#1{\texttt{#1}}\fi
\expandafter\ifx\csname urlprefix\endcsname\relax\def\urlprefix{URL }\fi
\expandafter\ifx\csname href\endcsname\relax
  \def\href#1#2{#2} \def\path#1{#1}\fi

\bibitem{fisher1936use}
R.~A. Fisher, The use of multiple measurements in taxonomic problems, Annals of
  eugenics 7~(2) (1936) 179--188.

\bibitem{rao1948utilization}
C.~R. Rao, The utilization of multiple measurements in problems of biological
  classification, Journal of the Royal Statistical Society. Series B
  (Methodological) 10~(2) (1948) 159--203.

\bibitem{fukunaga1990statistical}
K.~Fukunaga, Introduction to statistical pattern recognition (2nd ed.),
  Academic Press, 1990.

\bibitem{ye2005characterization}
J.~Ye, Characterization of a family of algorithms for generalized discriminant
  analysis on undersampled problems, Journal of Machine Learning Research
  6~(Apr) (2005) 483--502.

\bibitem{zhang2015sparse}
X.~Zhang, D.~Chu, R.~C. Tan, Sparse uncorrelated linear discriminant analysis
  for undersampled problems, IEEE Transactions on Neural Networks and Learning
  Systems 27~(7) (2015) 1469--1485.

\bibitem{wen2018robust}
J.~Wen, X.~Fang, J.~Cui, L.~Fei, K.~Yan, Y.~Chen, Y.~Xu, Robust sparse linear
  discriminant analysis, IEEE Transactions on Circuits and Systems for Video
  Technology 29~(2) (2018) 390--403.

\bibitem{Dornaika2020141}
F.~Dornaika, A.~Khoder, Linear embedding by joint robust discriminant analysis
  and inter-class sparsity, Neural Networks 127 (2020) 141--159.
\newblock \href {http://dx.doi.org/10.1016/j.neunet.2020.04.018}
  {\path{doi:10.1016/j.neunet.2020.04.018}}.

\bibitem{Zheng2021}
Z.~Zheng, H.~Sun, Y.~Zhou, Multiple discriminant analysis for collaborative
  representation-based classification, Pattern Recognition 112.
\newblock \href {http://dx.doi.org/10.1016/j.patcog.2021.107819}
  {\path{doi:10.1016/j.patcog.2021.107819}}.

\bibitem{hastie1996discriminant}
T.~Hastie, R.~Tibshirani, Discriminant analysis by gaussian mixtures, Journal
  of the Royal Statistical Society: Series B (Methodological) 58~(1) (1996)
  155--176.

\bibitem{sugiyama2007dimensionality}
M.~Sugiyama, Dimensionality reduction of multimodal labeled data by local
  fisher discriminant analysis, Journal of machine learning research 8~(May)
  (2007) 1027--1061.

\bibitem{lai2018robust}
Z.~Lai, N.~Liu, L.~Shen, H.~Kong, Robust locally discriminant analysis via
  capped norm, IEEE Access 7 (2018) 4641--4652.

\bibitem{zhu2006subclass}
M.~Zhu, A.~M. Martinez, Subclass discriminant analysis, IEEE Transactions on
  Pattern Analysis and Machine Intelligence 28~(8) (2006) 1274--1286.

\bibitem{gkalelis2011mixture}
N.~Gkalelis, V.~Mezaris, I.~Kompatsiaris, Mixture subclass discriminant
  analysis, IEEE Signal Processing Letters 18~(5) (2011) 319--322.

\bibitem{wan2017separability}
H.~Wan, H.~Wang, G.~Guo, X.~Wei, Separability-oriented subclass discriminant
  analysis, IEEE transactions on pattern analysis and machine intelligence
  40~(2) (2017) 409--422.

\bibitem{Chumachenko2021}
K.~Chumachenko, J.~Raitoharju, A.~Iosifidis, M.~Gabbouj, Speed-up and
  multi-view extensions to subclass discriminant analysis, Pattern Recognition
  111.
\newblock \href {http://dx.doi.org/10.1016/j.patcog.2020.107660}
  {\path{doi:10.1016/j.patcog.2020.107660}}.

\bibitem{martinez2005linear}
A.~M. Martinez, M.~Zhu, Where are linear feature extraction methods
  applicable?, IEEE Transactions on Pattern Analysis and Machine Intelligence
  27~(12) (2005) 1934--1944.

\bibitem{hamsici2008bayes}
O.~C. Hamsici, A.~M. Martinez, Bayes optimality in linear discriminant
  analysis, IEEE transactions on pattern analysis and machine intelligence
  30~(4) (2008) 647--657.

\bibitem{belhumeur1997eigenfaces}
P.~N. Belhumeur, J.~P. Hespanha, D.~J. Kriegman, Eigenfaces vs. fisherfaces:
  Recognition using class specific linear projection, IEEE Transactions on
  Pattern Analysis \& Machine Intelligence~(7) (1997) 711--720.

\bibitem{yang2004essence}
J.~Yang, Z.~Jin, J.-y. Yang, D.~Zhang, A.~F. Frangi, Essence of kernel fisher
  discriminant: Kpca plus lda, Pattern Recognition 37~(10) (2004) 2097--2100.

\bibitem{yang2005kpca}
J.~Yang, A.~F. Frangi, J.-y. Yang, D.~D. Zhang, Z.~Jin, Kpca plus lda: a
  complete kernel fisher discriminant framework for feature extraction and
  recognition, IEEE Transactions on pattern analysis and machine intelligence.

\bibitem{baudat2000generalized}
G.~Baudat, F.~Anouar, Generalized discriminant analysis using a kernel
  approach, Neural computation 12~(10) (2000) 2385--2404.

\bibitem{howland2004generalizing}
P.~Howland, H.~Park, Generalizing discriminant analysis using the generalized
  singular value decomposition, IEEE Transactions on Pattern Analysis \&
  Machine Intelligence~(8) (2004) 995--1006.

\bibitem{wang2007trace}
H.~Wang, S.~Yan, D.~Xu, X.~Tang, T.~Huang, Trace ratio vs. ratio trace for
  dimensionality reduction, in: 2007 IEEE Conference on Computer Vision and
  Pattern Recognition, IEEE, 2007, pp. 1--8.

\bibitem{jia2009trace}
Y.~Jia, F.~Nie, C.~Zhang, Trace ratio problem revisited, IEEE Transactions on
  Neural Networks 20~(4) (2009) 729--735.

\bibitem{Cao2019218}
M.~Cao, C.~Chen, X.~Hu, S.~Peng, Towards fast and kernelized orthogonal
  discriminant analysis on person re-identification, Pattern Recognition 94
  (2019) 218--229.
\newblock \href {http://dx.doi.org/10.1016/j.patcog.2019.05.035}
  {\path{doi:10.1016/j.patcog.2019.05.035}}.

\bibitem{ghiasi2019competitive}
K.~Ghiasi-Shirazi, Competitive cross-entropy loss: A study on training
  single-layer neural networks for solving nonlinearly separable classification
  problems, Neural Processing Letters 50~(2) (2019) 1115--1122.

\bibitem{ZareiGhiasi2022}
R.~Zarei-Sabzevar, K.~Ghiasi-Shirazi, A.~Harati, Prototype-based interpretation
  of the functionality of neurons in winner-take-all neural networks, IEEE
  Transactions on Neural Networks and Learning Systems (2022) 1--13\href
  {http://dx.doi.org/10.1109/TNNLS.2022.3155174}
  {\path{doi:10.1109/TNNLS.2022.3155174}}.

\bibitem{lu2005regularization}
J.~Lu, K.~N. Plataniotis, A.~N. Venetsanopoulos, Regularization studies of
  linear discriminant analysis in small sample size scenarios with application
  to face recognition, Pattern recognition letters 26~(2) (2005) 181--191.

\bibitem{juefei2016multi}
F.~Juefei-Xu, M.~Savvides, Multi-class fukunaga koontz discriminant analysis
  for enhanced face recognition, Pattern Recognition 52 (2016) 186--205.

\bibitem{samaria1994parameterisation_orl_dataset}
F.~S. Samaria, A.~C. Harter, Parameterisation of a stochastic model for human
  face identification, in: Proceedings of 1994 IEEE Workshop on Applications of
  Computer Vision, IEEE, 1994, pp. 138--142.

\bibitem{phillips1998feret}
P.~J. Phillips, H.~Wechsler, J.~Huang, P.~J. Rauss, The feret database and
  evaluation procedure for face-recognition algorithms, Image and vision
  computing 16~(5) (1998) 295--306.

\bibitem{phillips2000feret}
P.~J. Phillips, H.~Moon, S.~A. Rizvi, P.~J. Rauss, The feret evaluation
  methodology for face-recognition algorithms, IEEE Transactions on pattern
  analysis and machine intelligence 22~(10) (2000) 1090--1104.

\bibitem{lecun1998gradient_mnist_dataset}
Y.~LeCun, L.~Bottou, Y.~Bengio, P.~Haffner, et~al., Gradient-based learning
  applied to document recognition, Proceedings of the IEEE 86~(11) (1998)
  2278--2324.

\end{thebibliography}
\end{document}